\documentclass[nohyperref]{article}

\usepackage{graphicx}
\usepackage{booktabs}

\usepackage{times}
\usepackage{helvet}
\usepackage{courier}
\usepackage{graphicx}
\usepackage{caption} 

\usepackage[T1]{fontenc}    % use 8-bit T1 fonts
\usepackage{url}            % simple URL typesetting
\usepackage{booktabs}       % professional-quality tables
\usepackage{amsfonts}       % blackboard math symbols
\usepackage{nicefrac}       % compact symbols for 1/2, etc.
\usepackage{microtype}      % microtypography
\usepackage{xcolor}         % colors
\usepackage{comment}

\usepackage{makecell}
\usepackage{multirow}

\usepackage{algorithm}
\usepackage{algorithmic}

\usepackage{amsthm}
\usepackage{amsmath}
\usepackage{amsfonts}
\usepackage{amssymb}
\usepackage{dsfont}
\usepackage{mathtools}
\usepackage{bm}
\usepackage{xfrac}

\usepackage{caption}
\usepackage{subcaption}
\usepackage{thmtools,thm-restate}

\usepackage{macros}

\usepackage{footnote}
\usepackage{hyperref}

\usepackage{xspace}
\usepackage{mwe}
\usepackage{array}
\makesavenoteenv{tabular}
\makesavenoteenv{table}

\usepackage[accepted]{icml2023}

\usepackage[capitalize,noabbrev]{cleveref}

\theoremstyle{plain}
\icmltitlerunning{Robust One Class Classification using Signed Distance Function}

\begin{document}

\twocolumn[
\icmltitle{Robust One-Class Classification with Signed Distance Function\\ using 1-Lipschitz Neural Networks}



\icmlsetsymbol{equal}{*}

%\title{Certifiable Learning of Signed Distance Functions with 1-Lipschitz Neural Networks}

\begin{icmlauthorlist}
\icmlauthor{Louis B\'ethune}{equal,sch}
\icmlauthor{Paul Novello}{equal,yyy}
\icmlauthor{Guillaume Coiffier}{nge}
\icmlauthor{Thibaut Boissin}{yyy}
\icmlauthor{Mathieu Serrurier}{sch}
\icmlauthor{Quentin Vincenot}{comp}
\icmlauthor{Andres Troya-Galvis}{comp}
\end{icmlauthorlist}

\icmlaffiliation{yyy}{DEEL, IRT Saint Exupéry}
\icmlaffiliation{comp}{Thales Alénia Space}
\icmlaffiliation{sch}{IRIT, Université Paul Sabatier}
\icmlaffiliation{nge}{Université de Lorraine, CNRS, Inria, LORIA}

\icmlcorrespondingauthor{Louis Bethune}{louis.bethune@univ-toulouse.fr}

\icmlkeywords{signed distance function, Lipschitz neural network, robustness, one class learning, anomaly detection, neural implicit surface}

\vskip 0.3in
]



\printAffiliationsAndNotice{\icmlEqualContribution } % otherwise use the standard text.

%\begin{document}

\begin{abstract}
    We propose a new method, dubbed One Class Signed Distance Function (OCSDF), to perform One Class Classification (OCC) by provably learning the Signed Distance Function (SDF) to the boundary of the support of any distribution. The distance to the support can be interpreted as a normality score, and its approximation using 1-Lipschitz neural networks provides robustness bounds against $l2$ adversarial attacks, an under-explored weakness of deep learning-based OCC algorithms. As a result, OCSDF comes with a new metric, certified AUROC, that can be computed at the same cost as any classical AUROC. We show that OCSDF is competitive against concurrent methods on tabular and image data while being way more robust to adversarial attacks, illustrating its theoretical properties. Finally, as exploratory research perspectives, we theoretically and empirically show how OCSDF connects OCC with image generation and implicit neural surface parametrization.
\end{abstract}

\section{Introduction}

One class classification (OCC) is an instance of binary classification where all the points of the dataset at hand belong to the same (positive) class. The challenge of this task is to construct a decision boundary without using points from the other (negative) class. It has various safety-critical applications in anomaly detection, for instance to detect banking fraud, cyber-intrusion or industrial defect, in out-of-distribution detection, to prevent wrong decisions of Machine Learning models, or in Open-Set-Recognition. However, OCC algorithms suffer from limitations such as the \textbf{lack of negative data}, and \textbf{robustness issues} \cite{occrobust}, the latter being an under-explored topic in the OCC spectrum. Even though some algorithms do not use negative examples, many work cope with the lack of negative data with Negative Sampling, either artificially \cite{sipple_interpretable_2020} or using outlier exposure \cite{hendrycks_benchmarking_2019,fort_exploring_2021}. However, such samplings are often biased or heuristic. As for robustness, although some works design robust algorithms \cite{goyal_drocc_2020,lo_adversarially_2022}, it is always only empirically demonstrated \cite{hendrycks_benchmarking_2019}. 

\begin{figure*}[!t]
    \centering
    \includegraphics[width=0.9\linewidth, trim=0 0 0 0, clip]{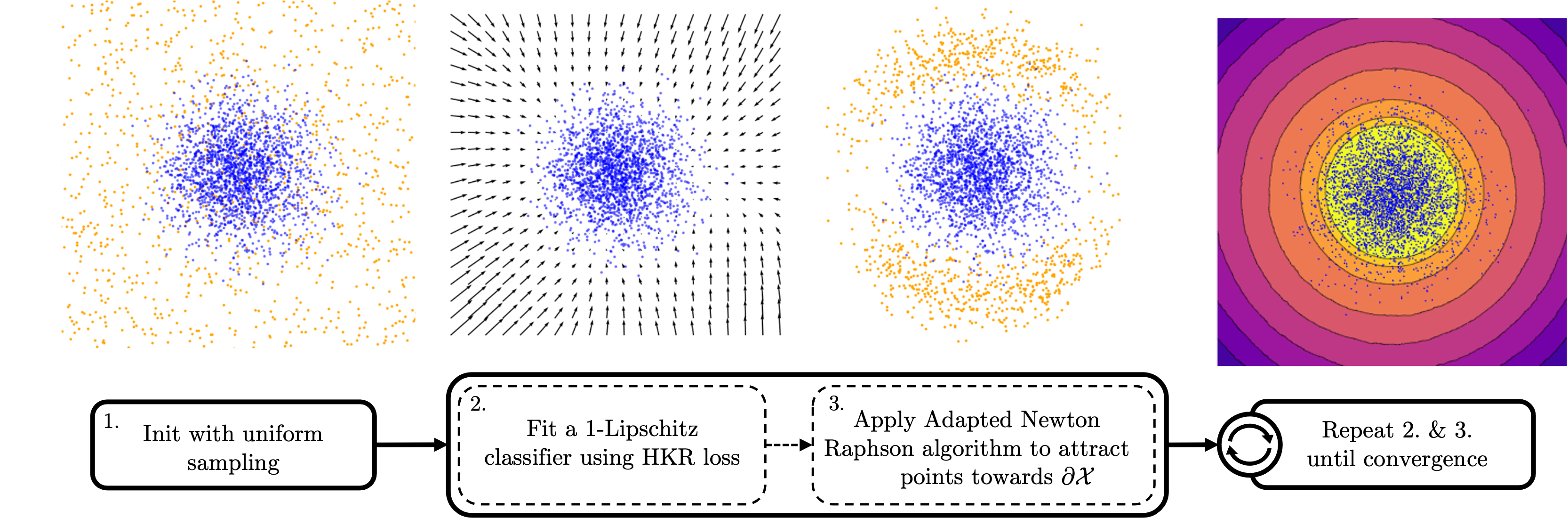}
    \caption{Summary of One Class Signed Distance Function (OCSDF). We start with an uniform negative sampling, then we fit a 1-Lipschitz classifier $f_{\theta}$ using the Hinge Kantorovich-Rubinstein loss. We apply the Adapted Newton Raphson algorithm \ref{alg:newtonraphson} to attract the points towards the boundary of the domain $\partial \mathcal{X}$ thanks to the smoothness of $f_{\theta}$, which in addition allows providing robustness certificates.}
    \label{fig:main}
    %\vspace{-0.3cm}
\end{figure*}

In this paper, we introduce a new framework to perform OCC based on the Signed Distance Function (SDF), a function traditionally used in computer graphics. Assume the positive samples are independently and identically obtained from a distribution $\Prob_X$ with compact support $\support\subset\Reals^d$. Let $\partial\support=\closure{\support}/\interior{\support}$ be the boundary of the distribution. The Signed Distance Function is the function $\Sdf:\Reals^d\rightarrow\Reals$:
\begin{equation}
\Sdf(x)=\begin{cases}
d(x,\partial\support)&\text{ if }x\in\support,\\
-d(x,\partial\support)&\text{ otherwise,}\\
\end{cases}
\end{equation}
where $d(x,\partial\support)=\inf_{z\in\partial\support} \|x-z\|_2$. The idea of our algorithm, which we call One Class Signed Distance Function (OCSDF) is to learn the SDF to the boundary of the positive data distribution and use it as a normality score. We show that the Hinge Kantorovich-Rubinstein (HKR) loss introduced by~\cite{serrurier2021achieving} allows provably learning the SDF with a 1-Lipschitz network. 

SDF exhibits desirable properties. First, by implicitly parametrizing the domain $\support$, it allows efficiently sampling points outside of $\support$ and performing principled Negative Sampling. Second, the SDF fulfils the Eikonal equation: $\|\nabla_x \Sdf(x)\|=1$. In particular, $\Sdf$ is 1-Lipschitz with respect to $l2$-nom : $\forall x,z\in\Reals^d, \|\Sdf(x)-\Sdf(z)\|_2\leq \|x-z\|_2$. This property provides exact robustness certificates for OCSDF in the form of a certified AUROC that can be computed at the same cost as AUROC. This regularity translates into solid empirical robustness as compared to other OCC baselines. In other words, OCSDF alleviates the \textbf{lack of negative data} and the \textbf{robustness issue}. We go further and highlight interesting research perspectives regarding OCSDF. Indeed, we show that learning the SDF with a 1-Lipschitz network enables a generative procedure that allows visualizing points at the boundary of $\support$. Moreover, It implicitly parametrizes the shape of $\support$, which connects One-Class Classification with implicit surface parametrization, intensively used in computer graphics for shape reconstruction. 

Our contributions are as follows. \textbf{(1)} We introduce a new OCC framework based on the Signed Distance Function to the boundary of the data distribution. We theoretically demonstrate that the SDF can be learned with a 1-Lipschitz neural net using the Hinge Kantorovich-Rubinstein (HKR) loss and Negative Sampling; \textbf{(2)} We evaluate the performances of OCSDF on several benchmarks and show its benefits for theoretical and empirical robustness; and \textbf{(3)} we demonstrate how OCSDF extends the applications of One Class Classification from traditional OOD detection to generative visualization and implicit surface parametrization for shape reconstruction from point clouds.

% \begin{enumerate}
%     \item We propose to parametrize ground truth $\Sdf$ with a 1-Lipschitz neural net thanks to Property~\ref{thm:sdf1lip}. The model is \textit{robust by design} thanks to the robustness certificates provided by 1-Lipschitz architectures. This gives formal guarantees against adversarial attacks, and allows to use the SDF in raymarching algorithms with correctness guarantees.   
%     \item Then, we show that the Hinge Kantorovich-Rubinstein (HKR) loss introduced by~\cite{serrurier2021achieving} allows to provably learn the SDF with a 1-Lipschitz network. This does not requires access to ground truth signed distance, unlike approaches based on least squares minimization. This allows to extend the use of SDF to high dimensional datasets (e.g images), not only 2D/3D like done ordinarily.
%     \item We illustrate the soundness of our approach on three benchmarks: supervised One Class learning on images data, unsupervised Anomaly Detection on tabular data, and Neural Implicit Surface learning from 3D point clouds.
% \end{enumerate}

% Finally, the link of HKR loss with optimal transport yields meaningful interpretations of the saliency maps of the model, in a fashion that reminds diffusion models.

%\vspace{-0.1cm}
\section{Related Work}

\paragraph{One Class Classification (OCC)} OCC is an instance of binary classification where all the points of the dataset at hand belong to the same (positive) class. The challenge of this task is to construct a decision boundary without using points from the other (negative) class. OCC amounts to finding a domain containing the support of the data distribution. That is why OCC is mainly used in Out Of Distribution (OOD), anomaly or novelty detection, with positive samples considered In Distribution (ID) and negative ones as OOD, anomalies or novelties. This task dates back to \cite{sager_iterative_1979, hartigan_estimation_1987} and was popularized for anomaly detection with One-class Support Vector Machines (OC-SVM)\cite{scholkopf_support_1999}. Since then, the field of OCC has flourished with many well-established algorithms such as Local Outlier Factors \cite{breuning_lof_2000}, Isolation Forests \cite{liu2008isolation} and their variants (see \cite{han_adbench_nodate} for a thorough benchmark). More recently, since Deep-SVDD \cite{ruff2018deep} - followed by several works such as \cite{bergman2019classification,golan_geometric_2018,goyal_drocc_2020,zenati_adversarially_2018,sabokrou2018adversarially} - Deep Learning has emerged as a relevant alternative to perform OCC due to its capacities to handle large dimensional data. However, methods of this field suffer from their lack of \textbf{robustness and certifications}, which makes them vulnerable to adversarial attacks. In addition, they always struggle to cope with the \textbf{lack of OOD data}. In this paper, we tackle these problems with an OCC algorithm based on approximating the SDF using \textbf{1-Lipschitz neural nets}. In addition, the SDF being intensively used in Computer Graphics, our algorithm establishes a new link between OCC and \textbf{implicit surface} parametrization.

%\vspace{-0.1cm}
\paragraph{SDF for neural implicit surfaces}

Historically, signed distance functions have been used in computer graphics to parametrize a surface as the level set of some function~\cite{novelloExploringDifferentialGeometry2022}. Given an incomplete or unstructured representation of a geometrical object (like a 3D point cloud or a triangle soup), recent methods aim at representing a smooth shape either as vectors in the latent space of a generative model~\cite{achlioptasLearningRepresentationsGenerative2018, ben-hamuMultichartGenerativeSurface2018, groueixAtlasNetPapierMAch2018, chouGenSDFTwoStageLearning2022} or directly as parameters of a neural net~\cite{parkDeepSDFLearningContinuous2019, atzmonSALSignAgnostic2020}.
The first method allows for easy shape interpolation, while the latter proved to be a more robust approach~\cite{daviesEffectivenessWeightEncodedNeural2021}. Those neural implicit surfaces alleviate both the problems related to memory requirements of voxel-based representations and the combinatorial nature of meshes, making them ideally suited for rendering using ray marching~\cite{hartSphereTracingGeometric1995} and constructive solid geometry. In those contexts, the constraint $\|\nabla_x f(x)\|\leq 1$ is necessary to guarantee the validity of the geometrical query while having $\|\nabla_x f(x)\|$ as close as possible to 1 allows for greedier queries and faster computation times. In practice, training an SDF requires a dataset $(p,d)$ of points $p \in \mathbb{R}^3$ with their corresponding signed distance $d$ to the desired surface. Computing those distances requires the existence and availability of a ground truth, which is not always the case. Moreover, training tends to be unstable in general, and special care is needed for most computer graphics applications~\cite{sharpSpelunkingDeepGuaranteed2022}. Our method can instead be trained to approximate a surface without prior knowledge of the distances and is provably robust.

%\vspace{-0.15cm}
\paragraph{1-Lipschitz neural nets}

As noticed in~\cite{bethune2022pay,brau2022robust} 1-Lipschitz neural nets ~\cite{stasiak2006fast,li2019orthogonal,su2021scaling} are naturally linked to the signed distance function. In particular, they are 1-Lipschitz, i.e. they fulfil $\|\nabla_x f(x)\|\leq 1$ on the whole input space. They boast a rich literature, especially for convolutional neural nets~\cite{gayer2020convolutional,wang2020orthogonal,liu2021convolutional,achour2021existence,li2019preventing,trockman2021orthogonalizing,singla2021skew}. These networks benefit from several appealing properties: they are not subject to exploding nor vanishing gradients~\cite{li2019preventing}, they generalize well~\cite{bartlett2019nearly,bethune2022pay}, and they are elegantly connected to optimal transport theory~\cite{arjovsky2017wasserstein,serrurier2021achieving}. 1-Lipschitz neural nets also benefit from certificates against $l2$-attacks~\cite{li2019preventing,NEURIPS2018_48584348}; hence the approximation of $\Sdf$ is robust against $l2$-adversarial attacks \textit{by design}.

%In this work we use the Deel-Lip library\footnote{\url{https://github.com/deel-ai/deel-lip} distributed under MIT license.}. They are based on the work of~\cite{anil2019sorting} that proved that universal approximation in the set of 1-Lipschitz function was possible with normalized matrices; they suggest to use orthogonal matrices in affine layers (i.e $W_i^TW_i=I$) and GroupSort activation function to counter vanishing gradient: $\text{GroupSort2}(x)_{2i,2i+1}=[\min{(x_{2i},x_{2i+1})},\max{(x_{2i},x_{2i+1})}]$. Details about the architecture are given in appendix.

%\vspace{-0.3cm}
\paragraph{Robustness and certification} While robustness comes with many aspects, this work focuses mainly on adversarial attacks \cite{szegedy2013}. Extensive literature explores the construction of efficient attacks \cite{Goodfellow2014a} \cite{Brendel2018} \cite{carlini2017towards}. As nearly any deep learning architecture is vulnerable, defenses have also been developed with notably adversarial training \cite{madry2017towards} \cite{zhang2019theoretically}\cite{shafahi2019adversarial}, or randomized smoothing \cite{cohen2019certified}\cite{carlini2022certified}. Since early works pointed out the link between the Lipschitz constant of a network and its robustness, Lipschitz-constrained networks have also been studied \cite{anil2019sorting} \cite{serrurier2021achieving}.
Similarly to classifiers, OCC Algorithms based on deep neural nets suffer from their natural weakness to adversarial attacks \cite{occrobust}. Although some works design robust algorithms \cite{goyal_drocc_2020,lo_adversarially_2022}, the robustness achieved is always only empirically demonstrated \cite{hendrycks_benchmarking_2019}. Few works provide theoretical certifications (we only found \cite{bitterwolf_certifiably_2020} based on interval bounds propagation). In this work, we leverage the properties of 1-Lipschitz networks to provide certifications.

%\vspace{-0.1cm}
\paragraph{Tackling the lack of OOD data} The previously mentioned OCC and OOD algorithms, as well as many others \cite{hendrycks_baseline_2018,hsu_generalized_2020-1} are designed to avoid the need for OOD data. However, some works aim at falling back to classical binary classification by artificially generating negative samples. The idea of Negative Sampling is not recent and appeared in \cite{forrest1994self} for detecting computer viruses or to emulate the distinction made by antibodies between pathogens and body cells \cite{gonzalez2002combining}. It has been introduced in anomaly detection by \cite{ayara_negative_2002} and studied by several works summarized in \cite{jinyin2011study}, but has lost popularity due to its practical inefficiency (e.g. compared to One-Class Support Vector Machines (OCSVM) \cite{stibor_is_2005}). Recently, some works revived the idea of using OOD data, either by artificial negative sampling \cite{lee_training_2018-1,sipple_interpretable_2020,goyal_drocc_2020,pourreza_g2d_2021}, or by using OOD data from other sources, a procedure called outlier exposure \cite{fort_exploring_2021,hendrycks_deep_2019}. However, outlier exposure suffers from bias since OOD data does not come from the same data space. Therefore, we follow the first idea and sample negative data points close to the domain $\mathcal{X}$, thanks to the orthogonal neural nets-based estimation of the SDF.

% This work focuses on these kind of networks for two reasons. First, it allow certifiable robustness, where the network provide robustness certificate for each predictions. Second, this method combine well with OCC, unlocking proofs \thib{todo: lister une preuve sympa nécéssitant le 1lip, genre preuve d'unicité/convergence par ex}

% \textcolor{red}{A VOUS DE JOUER leeeeezzzzzgoooo}

%\vspace{-0.1cm}
\section{Method}
%\vspace{-0.1cm}

The method aims to learn the Signed Distance Function (SDF) by reformulating the one-class classification of $\Prob_X$ as a binary classification of $\Prob_X$ against a carefully chosen distribution $Q(\Prob_X)$. We show that this formulation yields desirable properties, especially when the chosen classifier is a 1-Lipschitz neural net trained with the Hinge Kantorovich-Rubinstein (HKR) loss.
  
%\subsection{Metric One Class Learning as special case of Binary Classification}
%\vspace{-0.1cm}
\subsection{SDF learning formulated as binary classification}
%\vspace{-0.1cm}

 We formulate SDF learning as a binary classification that consists of classifying samples from $\Prob_X$ against samples from a complementary distribution, as defined below. 

\begin{definition}[$\bedd$ Complementary Distribution (informal)]\label{def:beddinformal}
Let $Q$ be a distribution of compact support included in $B$, with disjoint support from that of $\Prob_X$ that ``fills'' the remaining space, with $2\epsilon$ gap between $\support$ and $\supp Q$. Then we write $Q\bedd\Prob_X$. 
\end{definition}

A formal definition is given in Appendix \ref{app:pandc}. For image space with pixel intensity in $[0,1]$, we take $B=[0, 1]^{W\times H\times C}$. For tabular data, a hypercube of side length ten times the standard deviation of the data along the axis. A data point falling outside $B$ is trivially considered anomalous due to aberrant values.  Binary classification between $\Prob_X$ and any $Q\bedd\Prob_X$ allows the construction of the optimal signed distance function, using the Kantorovich-Rubinstein (HKR) Hinge loss~\cite{serrurier2021achieving}, thanks to the following theorem.

\begin{restatable}{theorem}{hkrsdf}{\normalfont\textbf{SDF Learning with HKR loss.}} Let $\hkr(yf(x))=\lambda\max{(0, m-yf(x))}-yf(x)$ be the Hinge Kantorovich Rubinstein loss, with margin $m=\epsilon$, regularization $\lambda>0$, prediction $f(x)$ and label $y\in\{-1,1\}$. Let $Q$ be a probability distribution on $B$. Let $\Empirical^{\text{hkr}}(f)$ be the population risk:
\begin{equation}
\begin{aligned}
    \Empirical^{\text{hkr}}(f,\Prob_X,Q)\defeq&\Expect_{x\sim\Prob_X}[\hkr(f(x))]\\
    &+\Expect_{z\sim Q}[\hkr(-f(z))].
\end{aligned}
\end{equation}
Let $f^{*}$ be the minimizer of population risk, whose existence is guaranteed with Arzel\`a-Ascoli theorem~\cite{bethune2022pay}:
\begin{equation}\label{eq:empiricalminimizer}
    f^{*}\in\arginf_{f\in\Lip}\Empirical^{\text{hkr}}(f,\Prob_X,Q),
\end{equation}
where $\Lip$ is the set of Lipschitz functions $\mathbb{R}^d \rightarrow \mathbb{R}$ of constant $1$. Assume that $Q\bedd\Prob_X$. \textbf{Then}, $f^{*}$ approximates the signed distance function over $B$:
\begin{equation}
    \begin{aligned}
        \forall x\in \support,\quad&\Sdf(x)= f^{*}(x)-m,\\
        \forall z\in\supp Q,\quad&\Sdf(z)= f^{*}(z)-m.
    \end{aligned}
\end{equation}\label{thm:hkrsdf}
Moreover, for all $x\in\supp Q\cup \support$:
$$\text{sign}(f(x))=\text{sign}(\Sdf(x)).$$  
\end{restatable}
%\begin{proof}Follows from the properties of HKR loss on separable distributions. Detail in appendix.\end{proof}

Note that if $m=\epsilon\ll 1$, then we have $f^{*}(x)\approx \Sdf(x)$. In this work, we parametrize $f$ as a 1-Lipschitz neural network, as defined below, because they fulfil $f\in\Lip$ by construction.

\begin{definition}[1-Lipschitz neural network (informal)]
Neural network with Groupsort activation function and orthogonal transformation in affine layers parameterized like in~\cite{anil2019sorting}.
\end{definition}

%In this work we use the Deel-Lip library\footnote{\url{https://github.com/deel-ai/deel-lip} distributed under MIT license.}. They are based on the work of~\cite{anil2019sorting} that proved that universal approximation in the set of 1-Lipschitz function was possible with normalized matrices; they suggest to use orthogonal matrices in affine layers (i.e $W_i^TW_i=I$) and GroupSort activation function to counter vanishing gradient: $\text{GroupSort2}(x)_{2i,2i+1}=[\min{(x_{2i},x_{2i+1})},\max{(x_{2i},x_{2i+1})}]$. Details about the architecture are given in appendix.

Details about the implementation can be found in Appendix \ref{app:ortho}. Theorem \ref{thm:hkrsdf} tells us that if we characterize the complementary distribution $Q$, we can approximate the SDF with a 1-Lipschitz neural classifier trained with HKR loss. We now need to find the complementary distribution $Q$.

%\vspace{-0.1cm}
\subsection{Finding the complementary distribution by targeting the boundary}\label{sec:findcompl}

We propose to seek $Q$ through an alternating optimization process: at every iteration $t$, a proposal distribution $Q_t$ is used to train a 1-Lipschitz neural net classifier $f_t$ against $\Prob_X$ by minimizing empirical HKR loss. Then, the proposal distribution is updated in $Q_{t+1}$ based on the loss induced by $f_t$, and the procedure is repeated. % We detail this approach in the next section. 

We suggest starting from the uniform distribution: $Q_0=\Uniform(B)$. Observe that in high dimension, due to the \textit{curse of dimensionality}, a sample $z\sim Q_0$ is unlikely to satisfy $z\in\support$. Indeed the data lies on a low dimensional manifold $\support$ for which the Lebesgue measure is negligible compared to $B$. Hence, in the limit of small sample size $n\ll\infty$, a sample ${Z_n\sim Q_0^{\otimes n}}$ fulfills ${Z_n\bedd\Prob_X}$. This phenomenon is called the Concentration Phenomenon and has already been leveraged in anomaly detection in \cite{sipple_interpretable_2020}. However, the \textit{curse} works both ways and yields a high variance in samples $Z_n$. Consequently, the variance of the associated minimizers $f_0$ of equation~\ref{eq:empiricalminimizer} will also exhibit a high variance, which may impede the generalization and convergence speed. Instead, the distribution $Q_t$ must be chosen to produce higher density in the neighborhood of the boundary $\partial\support$. The true boundary is unknown, but the level set $\Level_t=f_t^{-1}(\{-\epsilon\})$ of the classifier can be used as a proxy to improve the initial proposal $Q_0$. We start from $z_0\sim Q_0$, and then look for a displacement $\delta\in\Reals^d$ such that $z+\delta\in\Level_t$. To this end, we take inspiration from the multidimensional Newton-Raphson method and consider a linearization of $f_t$:
\begin{equation}
    f_t(z_0+\delta) \approx f_t(z_0) + \langle \nabla_x f_t(z_0),\delta \rangle.
\end{equation}
Since 1-Lipschitz neural nets with GroupSort activation function are piecewise \textit{affines}~\cite{tanielian2021approximating}, the linearization is locally exact, hence the following property.
\begin{property}
Let $f_t$ be a 1-Lipschitz neural net with GroupSort activation function. Almost everywhere on $z_0\in\Reals^d$, there exists $\delta_0>0$ such that for every $\|\delta\|\leq \delta_0$, we have:
\begin{equation}
    f_t(z_0+\delta) = f_t(z_0) + \langle \nabla_x f_t(z_0),\delta \rangle.
\end{equation}
\end{property}

Since $f_t(z_0+\delta)\in \Level_t$ translates into $f_t(z_0+\delta) = -\epsilon$,
\begin{equation}\label{eq:newton}
        \delta = -\frac{f_t(z_0)+\epsilon}{\|\nabla_x f_t(z_0)\|^2}\nabla_x f_t(z_0).
\end{equation}

Properties of $\hkr$ ensure that the optimal displacement follows the direction of the gradient $\nabla_x f_t(z_0)$, which coincides with the direction of an optimal transportation plan ~\cite{serrurier2021achieving}. The term $\|\nabla_x f_t(z_0)\|$ enjoys an interpretation as a Local Lipschitz Constant (see~\cite{jordan2020exactly}) of $f_t$ around $z_0$, which we know fulfills $\|\nabla_x f_t(z_0)\|\leq 1$ when parametrized with an 1-Lipschitz neural net. When $f_t$ is trained to perfection, the expression for $\delta$ simplifies to $\delta=-f_t(z_0)\nabla_x f_t(z_0)$ thanks to Property~\ref{prop:gnphkr}.  

\begin{property}[Minimizers of $\hkr$ are Gradient Norm Preserving (from~\cite{serrurier2021achieving})]\label{prop:gnphkr}
Let $f^{*}_t$ be the solution of Equation~\ref{eq:empiricalminimizer}. Then for almost every $z\in B$ we have $\|\nabla_x f^*_t(z)\|=1$.  
\end{property}

In practice, the exact minimizer $f^{*}_t$ is not always retrieved, but equation~\ref{eq:newton} still applies to imperfectly fitted classifiers. The final sample $z'\sim Q_t$ is obtained by generating a sequence of $T$ small steps to smooth the generation. The procedure is summarized in algorithm~\ref{alg:newtonraphson}. In practice, $T$ can be chosen very low (below $16$) without significantly hurting the quality of generated samples. Finally, we pick a random ``learning rate'' $\eta\sim\Uniform([0,1])$ for each negative example in the batch to ensure they distribute evenly on the path toward the boundary. The procedure also benefits from Property \ref{thm:fixpointstop_main}, which ensures that the distribution $Q_{t+1}$ obtained from $Q_t$ across several iterative applications of Algorithm \ref{alg:newtonraphson} still fulfils $Q^{t+1}\bedd\Prob_X$. The proof is given in Appendix \ref{app:sdfadv}.

\begin{property}[Complementary distributions are fix points]\label{thm:fixpointstop_main}
Let $Q^{t}$ be such that $Q^{t}\bedd\Prob_X$. Assume that $Q^{t+1}$ is obtained with algorithm~\ref{alg:mmalgo}. Then we have $Q^{t+1}\bedd\Prob_X$. 
\end{property}

\begin{algorithm}%[tb]
\caption{Adapted Newton–Raphson for Complementary Distribution Generation}
\label{alg:newtonraphson}
\textbf{Input}: 1-Lipschitz neural net $f_t$\\
\textbf{Parameter}: number of steps $T$\\
\textbf{Output}: sample $z'\sim Q_t(f)$
\begin{algorithmic}[1] %[1] enables line numbers
\STATE sample learning rate $\eta\sim\Uniform([0,1])$
\STATE $z_0\sim\Uniform(B)$  \COMMENT{ Initial approximation.}
\FOR{each step $t=1$ to $T$}
\STATE $z_{t+1}\gets z_t-\frac{\eta}{T}\frac{\nabla_x f(z^t)}{\|\nabla_x f(z^t)\|_2^2}(f(z_t)+\epsilon)$\COMMENT{Refining.}
\STATE $z_{t+1}\gets \Pi_B(z_{t+1})$\COMMENT{ Stay in feasible set.}
\ENDFOR
\STATE \textbf{return} $z_T$
\end{algorithmic}
\end{algorithm}

\begin{remark}\label{remark:stochasticity}
In high dimension $d\gg 1$, when $\|\nabla_x f_t(z)\|=1$ and Vol$(B)\gg \text{Vol}(\support)$ the samples obtained with algorithm~\ref{alg:newtonraphson} are approximately uniformly distributed in the levels of $f_t$. It implies that the density of $Q$ increases exponentially fast (with factor $d$) with respect to the value of $-|f_t(\cdot)|$. This mitigates the adverse effects of the curse of dimensionality.   
\end{remark}
  
This scheme of ``generating samples by following gradient in input space'' reminds diffusion models~\cite{ho2020denoising}, feature visualization tools~\cite{olah2017feature}, or recent advances in VAE~\cite{kuzina2022alleviating}. However, no elaborated scheme is required for the training of $f_t$: 1-Lipschitz networks exhibit smooth and interpretable gradients~\cite{serrurier2022adversarial} which allows sampling from $\support$ ``for free'' as illustrated in figure~\ref{fig:mnist_gen_main}.  

\begin{remark}\label{remark:metropolis}
A more precise characterization of $Q_t$ built with algorithm~\ref{alg:newtonraphson} can be sketched below. Our approach bares some similarities with the spirit of Metropolis-adjusted Langevin algorithm~\cite{grenander1994representations}. In this method, the samples of $p(x)$ are generated by taking the stationary distribution $x_{t\veryshortarrow\infty}$ of a continuous Markov chain obtained from the stochastic gradient step iterates 
\vspace{-0.2cm}
\begin{equation}
\vspace{-0.2cm}
x_{t+1}\gets x_t+\zeta\nabla_x \log{p(x)}+\sqrt{2\zeta}Z
\end{equation}
for some distribution $p(x)$ and $Z\sim\mathcal{N}(\bf 0,\bf 1)$. By choosing the level set $\epsilon=0$, and $p(x)\propto\indicator\{f(x)\leq 0\}\exp{(-\eta f^2(x))}$ the score function $\zeta\nabla_x \log{p(x)}$ is transformed into $\nabla_x f(x)|f(x)|$ with $\zeta=\frac{\eta}{T}$. Therefore, we see that the density decreases exponentially faster with the squared distance to the boundary $\partial\support$ when there are enough steps $T\gg 1$. In particular when $\support = \{0\}$ we recover $p(x)$ as the pdf of a standard Gaussian $\mathcal{N}(\bf 0,\bf 1)$. Although the similarity is not exact (e.g., the diffusion term $\sqrt{2\eta}Z$ is lacking, $T$ is low, $\eta\sim\mathcal{U}([0,1])$ is a r. v.), it provides interesting complementary insights on the algorithm.  
\end{remark}

%\vspace{-0.2cm} 
\subsection{Alternating minimization for SDF learning}
%\vspace{-0.1cm}

Each classifier $f_t$ does not need to be trained from scratch. Instead, the same architecture is kept throughout training, and the algorithm produces a sequence of parameters $\theta_t$ such that $f_t=f_{\theta_t}$. Each set of parameters $\theta_t$ is used as initialization for the next one $\theta_{t+1}$. Moreover, we only perform a low fixed number of parameter updates for each $t$ in a GAN fashion. The final procedure of OCSDF is shown in Figure \ref{fig:main} and detailed in algorithm~\ref{alg:mmalgolazy} of Appendix \ref{app:sdfadv}.

%\vspace{-0.1cm}
\section{Properties}

%\vspace{-0.1cm}
\subsection{Certificates against adversarial attacks}  
%\vspace{-0.1cm}

The most prominent advantage of 1-Lipschitz neural nets is their ability to produce certificates against adversarial attacks~\cite{szegedy2013}. Indeed, by definition we have $f(x+\delta)\in[f(x)-\|\delta\|, f(x)+\|\delta\|]$ for every example $x\in\support$ and every adversarial attack $\delta\in\Reals^d$. This allows bounding the changes in AUROC score of the classifier for every possible radius $\epsilon>0$ of adversarial attacks.

\begin{proposition}[certifiable AUROC]\label{certifauroc}
 Let $F_0$ be the cumulative distribution function associated with the negative classifier's prediction (when $f(x) \leq 0$), and $p_1$ the probability density function of the positive classifier's prediction (when $f(x) > 0$). Then, for any attack of radius $\epsilon > 0$, the AUROC of the attacked classifier $f_{\epsilon}$ can be bounded by
 
%https://stats.stackexchange.com/questions/180638/how-to-derive-the-probabilistic-interpretation-of-the-auc
%\vspace{-0.1cm}
\begin{equation}
%\vspace{-0.1cm}
    \text{AUROC}_{\epsilon}(f) = \int_{-\infty}^{\infty} F_0(t)p_1(t-2\epsilon)dt.
\end{equation}
\end{proposition}

The proof is left in Appendix \ref{app:certifauroc}. The certified AUROC score can be computed analytically without performing the attacks empirically, solely from score predictions $p_1(t-2\epsilon)$. More importantly, the certificates hold against \textit{any} adversarial attack whose $l_2$-norm is bounded by $\epsilon$, regardless of the algorithm used to perform such attacks. We emphasize that producing certificates is more challenging than traditional defence mechanisms (e.g, adversarial training, see~\cite{ijcai2021p591} and references therein) since they do not target defence against a specific attack method. Note that MILP solvers and branch-and-bound {approaches~\cite{tjengevaluating2019,wang2021beta}} can be used to improve the tightness of certificates, but at a higher cost.

%\vspace{-0.1cm}
\subsection{Rank normal and anomalous examples}
  
Beyond the raw certifiable AUROC score, $l_2$-based Lipschitz robustness enjoys another desirable property: the farther from the boundary the adversary is, the more important the attack budget required to change the score. In practice, it means that an attacker can only dissimulate anomalies already close to being ``normal''. Aberrant and hugely anomalous examples will require a larger (possibly impracticable) attack budget. This ensures that the \textit{normality score} is actually meaningful: it is high when the point is central in its local cloud, and low when it is far away.

%\vspace{-0.1cm}
\section{Experiments}
%\vspace{-0.1cm}

In this section, we evaluate the performances and properties of OCSDF on tabular and image data. All the experiments are conducted with tensorflow, and the 1-Lipschitz neural nets are implemented using the library Deel-Lip\footnote{\url{https://github.com/deel-ai/deel-lip}}. Our code is available at~\url{https://github.com/Algue-Rythme/OneClassMetricLearning}.
%\vspace{-0.2cm}
\subsection{Toy examples from Scikit-Learn}
%\vspace{-0.1cm}

We use two-dimensional toy examples from the Scikit-Learn library~\cite{pedregosa2011scikit}. Results are shown in figure~\ref{fig:toy2d}. The contour of the decision function are plotted in resolution $300\times 300$ pixels. The level sets of the classifier are compared against those of One Class SVM~\cite{scholkopf2001estimating} and Isolation Forest~\cite{liu2008isolation}. We also train a conventional network with Binary Cross Entropy against complementary distribution $Q_t$, and we show it struggles to learn a meaningful decision boundary. Moreover, its Local Lipschitz Constant~\cite{jordan2020exactly} increases uncontrollably, as shown in table~\ref{tab:llc_main}, which makes it prone to adversarial attacks. Finally, there is no natural interpretation of the prediction of the conventional network in terms of distance: the magnitude $|f(\cdot)|$ of the predictions quickly grows above $10^3$, whereas for 1-Lipschitz neural nets, it is approximately equal to the signed distance function $\Sdf$. We refer to Appendix \ref{app:toy2d} for visualizations.

\begin{figure}[!ht]
    \centering
    
    \begin{subfigure}{0.02\linewidth}
    \rotatebox{90}{\small \textbf{OCSDF}} 
    \end{subfigure}
    \begin{subfigure}{0.28\linewidth}
        \centering
        \includegraphics[width=1.\linewidth]{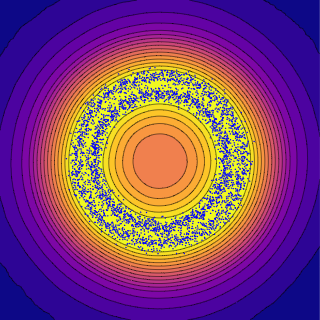}
    \end{subfigure}
    \begin{subfigure}{0.28\linewidth}
        \centering
        \includegraphics[width=1.\linewidth]{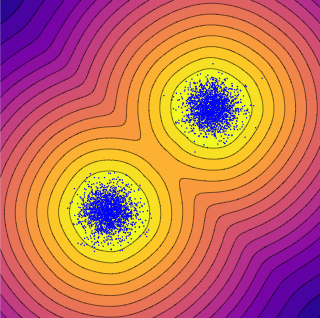}
    \end{subfigure}
    \begin{subfigure}{0.28\linewidth}
        \centering
        \includegraphics[width=1.\linewidth]{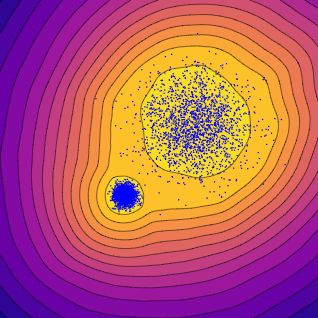}
    \end{subfigure}\\
        \begin{subfigure}{0.02\linewidth}
    \rotatebox{90}{\small \textbf{OCSVM}} 
    \end{subfigure}
    \begin{subfigure}{0.28\linewidth}
        \centering
        \includegraphics[width=1.\linewidth]{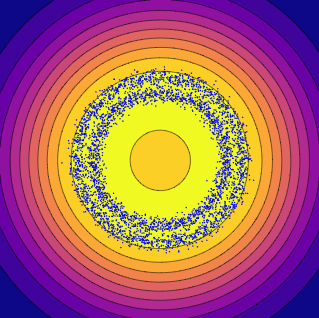}
        \subcaption{Two circles.}
    \end{subfigure}
    \begin{subfigure}{0.28\linewidth}
        \centering
        \includegraphics[width=1.\linewidth]{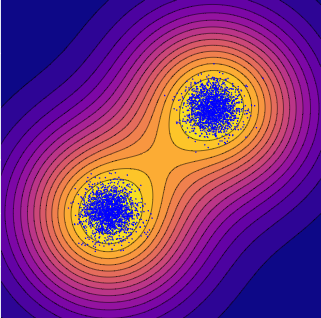}
        \subcaption{Two blobs.}
    \end{subfigure}
    \begin{subfigure}{0.28\linewidth}
        \centering
        \includegraphics[width=1.\linewidth]{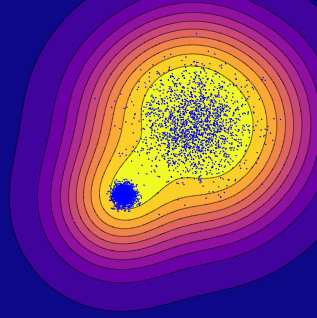}
        \subcaption{Blob \& cloud.}
    \end{subfigure}
    \caption{Contour plots of our method with 1-Lipschitz (LIP) network and $\hkr$ (HKR) loss on toy examples of Scikit-learn.}\label{fig:toy2d}
    %\vspace{-0.5cm}
    %Datasets: (a) One blob (b) Two circles (c) Two blobs (d) Two unbalanced blobs (e) Two moons.
\end{figure}

%\vspace{-0.1cm}
\subsection{Anomaly Detection on Tabular datasets}

We tested our algorithm on some of the most prominent anomaly detection benchmarks of ODDS library~\cite{Rayana:2016}. In this unsupervised setting (like ADBench~\cite{han_adbench_nodate}) all the examples (normal examples and anomalies) are seen during training, but their true label is unknown. To apply our method, the only hyperparameter needed is the margin $m$ that we select in the range $[0.01, 0.05, 0.2, 1.]$. For each value, the results are averaged over $20$ independent runs train/test splits. Following ADBench guidelines and best practices from the AD community, we only compute the AUROC, since this metric is symmetric under label flip. We report the best average in table~\ref{tab:tabularrocad} along baselines from ADBench~\cite{han_adbench_nodate}. As observed by~\cite{han_adbench_nodate}, none of the algorithms clearly dominates the others, because what is considered an anomaly (or not) depends on the context. Among 14 other methods tested, our algorithm ranks $7.1 \pm 3.6 / 15$, while the best (Isolation Forests) ranks $4.5 \pm 3.2 / 15$. The experiment shows that our algorithm is competitive with respect to other broadly used baselines. Nonetheless, it brings several additional advantages. First, our algorithm can be seen as a parametric version of kNN for the euclidean distance, which leverages deep learning to avoid the costly construction of structures like a KDTree~\cite{maneewongvatana1999analysis} and the quadratic cost of nearest neighbor search, thereby enabling its application in high dimensions. Second, it provides robustness certificates. We illustrate those two advantages more thoroughly in the next section.

%\vspace{-0.1cm}
\begin{table}[!ht]
    \small
    \centering
    \resizebox{\linewidth}{!}{\begin{tabular}{|c|ccccc|}
        \hline
       \multirow{2}{*}{\makecell{Local Lipschitz\\ Constant}} & \makecell{One\\ Cloud}  & \makecell{Two\\ Clouds} & \makecell{Two\\ Blobs} & \makecell{Blob\\ Cloud} & \makecell{Two\\ Moons} \\
       \cline{2-6}
       &   26.66    &    122.84   &  1421.41  &      53.90     &   258.73  \\ 
       \hline
    \end{tabular}}
    \caption{Lower bound on the Local Lipschitz Constant (LLC) of conventional network after $10,000$ training steps for each toy example. It is the maximum of $\|\nabla_{x_i}f(x_i)\|$ over the train set.}
    \label{tab:llc_main}
    %\vspace{-0.2cm}
\end{table}
\begin{table*}%[]
    \centering
    \small
    %\hspace{-0.35cm}
    \resizebox{\textwidth}{!}{\begin{tabular}{cccc|ccccccc}
    \hline
      Dataset & $\bm{d}$ &\#no.+an. & perc. & \makecell{OCSDF\\ (Ours)} & \makecell{Deep\\ SVDD} & \makecell{OC\\ SVM} & IF & PCA & kNN & SOTA\\
    \hline
    \hline
    % Arrhythmia & 274 & 386+66 & $14.6\%$ & $77.2\pm 0.7$ & N/A & N/A & N/A & N/A & N/A & $81.7$ (ICL)\\
    \multicolumn{4}{c}{Average Rank} & $7.1\pm 3.6$ & $11.2$ & $7.5$ & $4.5$ &  $5.7$ & $7.8$ & $4.5\pm 3.2$ (IF)\\
    \hline  
    breastw & 9 & 444+239 & $35\%$ & $(\#10)\Hquad 82.6\pm 5.9$ & $65.7$ & $80.3$ & $98.3$ & $95.1$ & $97.0$ & $99.7$ (COPOD) \\ 
    cardio & 21 & 1,655+176 & $9.6\%$ & $(\#2)\Hquad 95.0\pm 0.1$ & $59.0$ & $93.9$ & $93.2$ & $95.5$ & $76.6$ & $95.5$ (PCA)\\
    glass & 9 & 205+9 & $4.2\%$ & $(\#7)\Hquad 73.9\pm 4.1$ & $47.5$ & $35.4$ & $77.1$ & $66.3$ & $82.3$ & $82.9$ (CBLOF)\\
    http (KDDCup99) & 3 & 565,287+2,211 & $0.4\%$ & $(\#11)\Hquad 67.5\pm 37$ & $69.0$ & $99.6$ & $99.96$ & $99.7$ & $03.4$ & $99.96$ (IF) \\
    Ionosphere & 33 & 225+126 & $36\%$ & $(\#7)\Hquad 80.2\pm 0.1$ & $50.9$ & $75.9$ & $84.5$ & $79.2$ & $88.3$ & $90.7$ (CBLOF)\\
    Lymphography & 18 & 142+6 & $4.1\%$ & $(\#8)\Hquad 96.1\pm 4.9$ & $32.3$ & $99.5$ & $99.8$ & $99.8$ & $55.9$ & $99.8$ (CBLOF)\\ 
    mammography & 6 & 10,923+260 & $2.32\%$ & $(\#6)\Hquad 86.0\pm 2.5$ & $57.0$ & $84.9$ & $86.4$ & $88.7$ & $84.5$ & $90.7$ (ECOD)\\
    musk & 166 & 2,965+97 & $3.2\%$ & $(\#8)\Hquad 92.6\pm 20.$ & $43.4$ & $80.6$ & $99.99$ & $100.0$ & $69.9$ & $100.0$ (PCA) \\
    Optdigits & 64 & 5,066+150 & $3\%$ & $(\#12)\Hquad 51.0\pm 0.9$ & $38.9$ & $54.0$ & $70.9$ & $51.7$ & $41.7$ & $87.5$ (CBLOF) \\ 
    Pima & 8 & 500+268 & $35\%$ & $(\#12)\Hquad 60.7\pm 1.0$ & $51.0$ & $66.9$ & $72.9$ & $70.8$ & $73.4$ & $73.4$ (kNN) \\
    satimage-2 & 36 & 5,732+71 & $1.2\%$ & $(\#3)\Hquad 97.9\pm 0.4$ & $53.1$ & $97.3$& $99.2$ & $97.6$ & $92.6$ & $99.8$ (CBLOF)\\
    Shuttle & 9 & 45,586+3,511 & $7\%$ & $(\#4)\Hquad 99.1\pm 0.3$ & $52.1$ & $97.4$ & $99.6$ & $98.6$ & $69.6$ & $99.6$ (IF) \\ 
    smtp (KDDCup99) & 3 & 95,126+30 & $0.03\%$ & $(\#4)\Hquad 87.1\pm 3.5$ & $78.2$ & $80.7$ & $89.7$ & $88.4$ & $89.6$ & $89.7$ (IF) \\
    speech & 400 & 3,625+61 & $1.65\%$ & $(\#15)\Hquad 46.0\pm 0.2$ & $53.4$ & $50.2$ & $50.7$ & $50.8$ & $51.0$ & $56.0$ (COF)\\
    thyroid & 6 & 3,679+93 & $2.5\%$ & $(\#5)\Hquad 95.9\pm 0.0$ & $49.6$ & $87.9$ & $98.3$ & $96.3$ & $95.9$ & $98.3$ (IF)\\
    vertebral & 6 & 210+30 & $12.5\%$ &$(\#4)\Hquad 48.6\pm 2.6$ & $36.7$ & $38.0$ & $36.7$ & $37.0$ & $33.8$ & $53.2$ (DAGMM)\\
    vowels & 12 & 1,406+50 & $3.4\%$ & $(\#2)\Hquad 94.7\pm 0.7$ & $52.5$ & $61.6$ & $73.9$ & $65.3$ & $97.3$ & $97.3$ (kNN)\\
    WBC & 30 & 357+21 & $5.6\%$ & $(\#10)\Hquad 93.6\pm 0.1$ & $55.5$ & $99.0$ & $99.0$ & $98.2$ & $90.6$ & $99.5$ (CBLOF) \\ 
    Wine & 13 & 119+10 & $7.7\%$ & $(\#5)\Hquad 81.5\pm 0.9$ & $59.5$ & $73.1$ & $80.4$ & $84.4$ & $45.0$ & $91.4$ (HBOS) \\ 
    %\hline
    \end{tabular}}
    \caption{AUROC score for tabular data, averaged over $20$ runs. The dimension of the dataset is denoted by $\bm{d}$. In the \textbf{Anomaly Detection protocol (AD)} we use all the data (normal class and anomalies) for training, in an unsupervised fashion. The ``\#no.+an.'' column indicates part of normal (no.) and anomalous (an.) data used during training for each protocol. SOTA denominates the best score ever reported on the dataset, obtained by crawling relevant literature, or ADBench~\cite{han_adbench_nodate} results (table D4 page 37). We report the rank as (\#rank) among 14 other methods.\\
    }
    \label{tab:tabularrocad}
    % arrhythmia: ANOMALY DETECTION FOR TABULAR DATA WITH INTERNAL CONTRASTIVE LEARNING (no name => chose ICL)~\cite{shenkar2022anomaly}
    %\footnote{Misleading average: half of the runs yield same AUROC as kNN, and the other half around 99\% AUROC.}
\end{table*}

\begin{table*}%[]
    \centering
    %\hspace{-0.35cm}
    \small
    \resizebox{\textwidth}{!}{\begin{tabular}{c|ccccc|ccc}
    \hline
      MNIST & \makecell{OCSDF} & OCSDF & OCSDF & OCSDF & OCSDF & \makecell{OC\\ SVM} & \makecell{Deep\\ SVDD} & IF \\
      Certificates & $\epsilon=0$ & $\epsilon=8/255$ & $\epsilon=16/255$ & $\epsilon=36/255$ & $\epsilon=72/255$ & $\epsilon=0$ & $\epsilon=0$ & $\epsilon=0$\\
    \hline
    %\hline
    mAUROC &  $\bf 95.5\pm 0.4$ & $93.2\pm 2.1$ & $89.9\pm 3.5$ & $78.4\pm 6.4$ & $57.5\pm 7.5$ & $91.3\pm 0.0$ & $\bf 94.8\pm 0.9$ & $92.3\pm 0.5$\\
    \hline
    digit 0 & $\bf 99.7\pm 0.1$ & $99.6\pm 0.2$ & $99.5\pm 0.2$ & $99.0\pm 0.6$ & $96.2\pm 3.0$ & $98.6\pm 0.0$ & $98.0\pm 0.7$ & $98.0\pm 0.3$\\
    digit 1 & $\bf 99.8\pm 0.0$ & $99.7\pm 0.0$ & $99.6\pm 0.1$ & $99.2\pm 0.3$ & $96.2\pm 1.6$ & $99.5\pm 0.0$ & $\bf 99.7\pm 0.1$ & $97.3\pm 0.4$\\
    digit 2 & $\bf 90.6\pm 2.0$ & $85.3\pm 1.9$ & $78.2\pm 2.3$ & $53.1\pm 5.2$ & $14.1\pm 4.6$ & $82.5\pm 0.1$ & $\bf 91.7\pm 0.1$ & $88.6\pm 0.5$\\
    digit 3 & $\bf 93.4\pm 1.2$ & $90.0\pm 1.7$ & $85.0\pm 2.3$ & $66.2\pm 4.6$ & $26.9\pm 5.0$ & $88.1\pm 0.0$ & $91.9\pm 1.5$ & $89.9\pm 0.4$\\
    digit 4 & $\bf 96.5\pm 0.9$ & $95.3\pm 1.2$ & $93.9\pm 1.7$ & $89.4\pm 3.6$ & $76.2\pm 9.8$ & $94.9\pm 0.0$ & $\bf 94.9\pm 0.8$ & $92.7\pm 0.6$\\
    digit 5 & $\bf 93.9\pm 2.2$ & $89.0\pm 3.2$ & $81.6\pm 4.7$ & $54.0\pm 8.7$ & $15.6\pm 6.9$ & $77.1\pm 0.0$ & $88.5\pm 0.9$ & $85.5\pm 0.8$\\
    digit 6 & $\bf 98.7\pm 0.6$ & $98.1\pm 0.7$ & $97.2\pm 0.9$ & $93.1\pm 2.6$ & $74.9\pm 10.4$ & $96.5\pm 0.0$ & $\bf 98.3\pm 0.5$ & $95.6\pm 0.3$\\
    digit 7 & $\bf 97.1\pm 0.6$ & $96.5\pm 0.5$ & $95.6\pm 0.6$ & $92.2\pm 0.8$ & $81.2\pm 1.7$ & $93.7\pm 0.0$ & $94.6\pm 0.9$ & $92.0\pm 0.4$\\
    digit 8 & $89.4\pm 2.6$ & $83.3\pm 5.1$ & $74.7\pm 9.0$ & $50.3\pm 15.9$ & $24.4\pm 14.0$ & $88.9\pm 0.0$ & $\bf 93.9\pm 1.6$ & $89.9\pm 0.4$\\
    digit 9 & $\bf 96.4\pm 0.3$ & $95.3\pm 0.9$ & $93.8\pm 1.3$ & $87.8\pm 3.1$ & $68.9\pm 7.6$ & $93.1\pm 0.0$ & $\bf 96.5\pm 0.3$ & $93.5\pm 0.3$\\
    %\hline
    \end{tabular}}
    \begin{tabular}{c|ccccc|ccc}
    \hline
      CIFAR10 & \makecell{OCSDF} & OCSDF & OCSDF & OCSDF & OCSDF & \makecell{OC\\ SVM} & \makecell{Deep\\ SVDD} & IF \\
      Certificates & $\epsilon=0$ & $\epsilon=8/255$ & $\epsilon=16/255$ & $\epsilon=36/255$ & $\epsilon=72/255$ & $\epsilon=0$ & $\epsilon=0$ & $\epsilon=0$\\
    \hline
    %\hline
    mAUROC & $57.4\pm 2.1$ & $53.1\pm 2.1$ & $48.8\pm 2.1$ & $38.4\pm 1.9$ & $22.5\pm 1.4$ & $64.8\pm 8.0$ & $64.8\pm 6.8$ & $55.4\pm 8.0$\\
    \hline
    Airplane & $\bf 68.2\pm 4.5$ & $64.3\pm 3.9$ & $60.1\pm 3.2$ & $49.4\pm 1.1$ & $31.2\pm 3.6$ & $61.6\pm 0.9$ & $61.7\pm 4.1$ & $60.1\pm 0.7$\\
    Automobile & $57.3\pm 1.7$ & $52.5\pm 3.0$ & $47.6\pm 4.2$ & $36.1\pm 6.8$ & $19.8\pm 7.8$ & $\bf 63.8\pm 0.6$ & $\bf 65.9\pm 2.1$ & $50.8\pm 0.6$\\
    Bird & $\bf 51.8\pm 2.7$ & $47.5\pm 1.8$ & $43.2\pm 1.6$ & $33.4\pm 3.4$ & $19.5\pm 5.7$ & $\bf 50.0\pm 0.5$ & $\bf 50.8\pm 0.8$ & $\bf 49.2\pm 0.4$\\
    Cat & $\bf 58.8\pm 1.2$ & $54.6\pm 0.8$ & $50.3\pm 0.8$ & $40.0\pm 1.5$ & $24.4\pm 2.3$ & $55.9\pm 1.3$ & $\bf 59.1\pm 1.4$ & $55.1\pm 0.4$\\
    Deer & $49.4\pm 2.4$ & $45.3\pm 2.1$ & $41.4\pm 1.9$ & $32.2\pm 1.5$ & $18.8\pm 1.4$ & $\bf 66.0\pm 0.7$ & $60.9\pm 1.1$ & $49.8\pm 0.4$\\
    Dog & $56.3\pm 0.6$ & $51.9\pm 1.0$ & $47.5\pm 1.6$ & $36.7\pm 2.9$ & $20.6\pm 4.0$ & $62.4\pm 0.8$ & $\bf 65.7\pm 2.5$ & $58.4\pm 0.5$\\
    Frog & $52.6\pm 1.8$ & $48.7\pm 1.7$ & $44.9\pm 1.6$ & $35.8\pm 1.4$ & $22.4\pm 1.1$ & $\bf 74.7\pm 0.3$ & $67.7\pm 2.6$ & $42.9\pm 0.6$\\
    Horse & $49.5\pm 0.9$ & $45.5\pm 1.0$ & $41.5\pm 1.2$ & $32.5\pm 1.5$ & $18.8\pm 1.6$ & $62.6\pm 0.6$ & $\bf 67.3\pm 0.9$ & $55.1\pm 0.7$\\
    Ship & $68.6\pm 1.8$ & $64.6\pm 1.4$ & $60.4\pm 1.3$ & $49.3\pm 2.4$ & $29.8\pm 4.9$ & $\bf 74.9\pm 0.4$ & $\bf 75.9\pm 1.2$ & $\bf 74.2\pm 0.6$\\
    Truck & $61.3\pm 3.4$ & $56.5\pm 2.1$ & $51.5\pm 1.1$ & $39.0\pm 3.9$ & $20.0\pm 7.0$ & $\bf 75.9\pm 0.3$ & $73.1\pm 1.2$ & $58.9\pm 0.7$\\
    %\hline
    \end{tabular}
    \caption{AUROC score on the test set of MNIST and CIFAR10 in a \textit{one versus all} fashion, averaged on $10$ runs. We also report the AUROC of DeepSVDD~\cite{ruff2018deep} for completeness, along with the other AUROC scores of Isolation Forest (IF) and One-Class SVM (OC-SVM) reported in~\cite{ruff2018deep}. When the differences between some
    methods are not statistically significant, we highlight both.  When the confidence intervals overlap, we highlight both. We also show the \textbf{certifiable} AUROC against l-$2$ attacks of norms $\epsilon\in\{8/255,16/255,36/255\}$. Concurrent methods cannot provide certificates for $\epsilon>0$.}
    \label{tab:mnistroc}
\end{table*}

% \begin{table}%[]
%     \centering
%     %\hspace{-0.35cm}
%     \small
%     \begin{tabular}{c|cc}
%     \hline
%       \makecell{Methods\\} & \makecell{CSDFL\\ (ours)} & \makecell{DeepSDF\\ \cite{park2019deepsdf}}\\
%     \hline
%     %\hline
%     Target & \makecell{Generate $Q\bedd P$} & \makecell{Compute $\Sdf(x)$ from $P$}\\
%     Cost         & \makecell{Backward pass} & \makecell{Nearest Neighbor search} \\
%     Loss         & HKR $\hkr$ & $\|f(x)-\Sdf(x)\|^2_2$\\
%     Guarantees   & $f$ is 1-Lipschitz & None\\
%     \end{tabular}
%     \caption{Comparison of our approach against DeepSDF.}
%     \label{tab:comparisondeepsdf}
% \end{table}

%Nonetheless, on F1-score our method beat some significant previous works (see Appendix~\ref{tab:againsthrn}). 

%\vspace{-0.1cm}
\subsection{One Class learning on images}
%\vspace{-0.1cm}

We evaluate the performances of OCSDF for OCC, where only samples of the normal class are supposed to be available. To emulate this setting, we train a classifier on each of the classes of MNIST and Cifar10, and evaluate it on an independent test set in a \textit{one-versus-all} fashion. We compare our method against DeepSVDD \cite{ruff2018deep}, OCSVM \cite{scholkopf_support_1999}, and Isolation Forests \cite{liu2008isolation}. Details on the implementation of the baselines can be found in Appendix \ref{app:occimage}. The mean AUROC score is reported in table~\ref{tab:mnistroc} and averaged over $20$ runs. It is computed between the $1,000$ test examples of the target class and the remaining $9,000$ examples from other classes of the test set (both unseen during training). The histograms of normality scores are given in Appendix \ref{app:occimage}. OCSDF is competitive with respect to other baselines. In addition, it comes with several advantages described in the following.

%Note that the \textit{out-of-distribution} examples are not seen during training, but more importantly, the \textit{in-distribution} examples from the test set are not seen either. Hence the task evaluates both the generalization capacity (new example from the \textit{in-distribution}) and the discriminative capacity (against \textit{out-of-distribution}). This setting is more challenging because of the curse of dimensionality. 

%Since the models were not available we had to retrain them from scratch. Unfortunately there were discrepancies between the training protocol and the architecture described in original's paper~\cite{ruff2018deep} (LeNet), its official Pytorch repository~\footnote{See \url{https://github.com/lukasruff/Deep-SVDD-PyTorch}.} (LeNet + batch normalization), and its Tensorflow re-implementation from PyOD~\cite{zhao2019pyod} (no convolutions and l$2$-activation regularization). Hence, to allow fair comparison, we run all the versions. We compute the average AUROC over 10 runs, and we keep the best of those runs to compute the empirical AUROC against l$2$-PGD attacks, with the same implementation as Foolbox~\cite{rauber2017foolboxnative}.

%\vspace{-0.1cm}
\subsubsection{Certifiable and empirical robustness}
%\vspace{-0.1cm}

% For each class, we compute the certified AUROC score of each digit for $l_2$ attacks of radii $\epsilon\in\{8/255,16/255,36/255\}$.
None of the concurrent methods can provide certificates against $l_2$ attacks: in the work of~\cite{goyal_drocc_2020} the attacks are performed empirically (no certificates) with $l_\infty$ radii. In table~\ref{tab:mnistroc}, we report our certifiable AUROC with various radii $\epsilon\in\{0, 8/25, 16/255, 36/255, 72/255\}$. In figure~\ref{fig:robust} we report the empirical AUROC against $l_2$-PGD attacks with three random restarts, using stepsize $\zeta=0.025\epsilon$ like Foolbox~\cite{rauber2017foolboxnative}. These results illustrate our method's benefits: not only does it come with robustness certificates that are verified empirically, but the empirical robustness is also way better than DeepSVDD, especially for Cifar10. Note that for 1-Lipschitz network trained with $\hkr$ loss, all the attacks tend to find the same adversaries~\cite{serrurier2021achieving} - hence PGD is also representative of the typical score that would have been obtained with other attack methods.

\begin{figure}[!ht]
%\vspace{-0.2cm}
    \centering
    \begin{subfigure}{0.22\textwidth}
        \centering
        \includegraphics[scale=0.2]{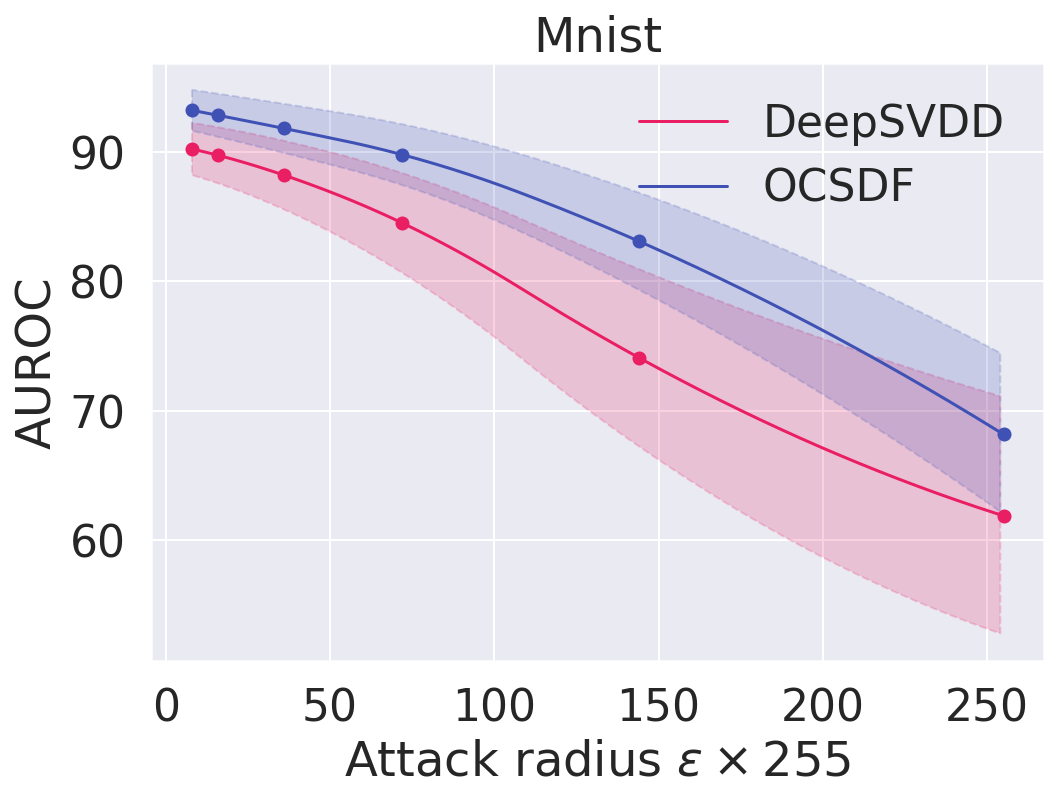}
    \end{subfigure}
    \begin{subfigure}{0.22\textwidth}
        \centering
        \includegraphics[scale=0.2]{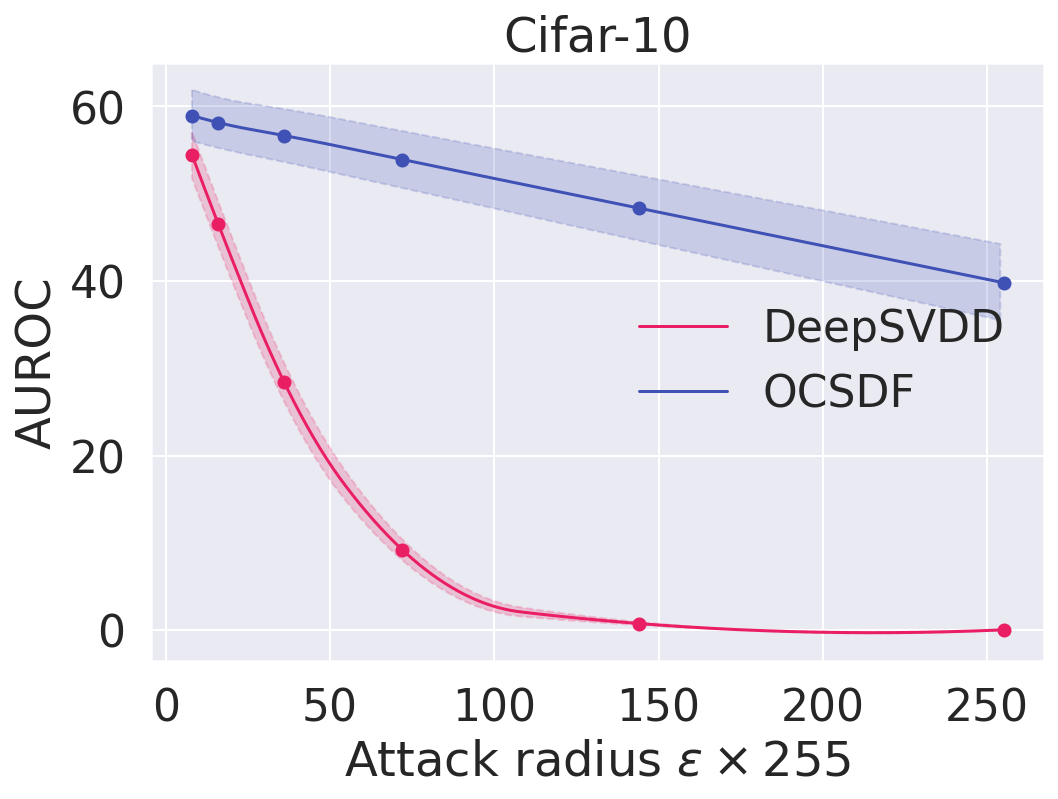}
    \end{subfigure}
    \caption{Empirical Mean AUROC on all classes against adversarial attacks of various radii in One Class setting, using default parameters of FoolBox~\cite{rauber2017foolboxnative}.}
    %\vspace{-0.15cm}
    \label{fig:robust}
\end{figure}
%Note that the certificates are generally pessimistic, and Eikhonal equation gap $1-\|\nabla_x f(x)\|_2^2$ translates into larger difference between certificates and empirical robustness.

%\vspace{-0.1cm}
\subsubsection{Visualization of the support}
%\vspace{-0.1cm}

OCSDF can be seen as a parametric version of kNN, which enables this approach in high dimensions. As a result, the decision boundary learned by the classifier can be materialized by generating adversarial examples with algorithm~\ref{alg:newtonraphson}. The forward computation graph is a classifier based on optimal transport, and the backward computation graph is an image generator. Indeed, the back-propagation through a convolution is a transposed convolution, a popular layer in the generator of GANs. Overall, the algorithm behaves like a WGAN~\cite{arjovsky2017wasserstein} with a single network fulfilling both roles. This unexpected feature opens a path to the explainability of the One Class classifier: the support learned can be visualized without complex feature visualization tools. In particular, it helps identify failure modes.
 
 %\vspace{-0.1cm}
 \begin{figure}[!ht]
    \centering
    
    %% 0 1
    \begin{subfigure}{0.24\textwidth}
        \centering
        \includegraphics[width=1.\textwidth]{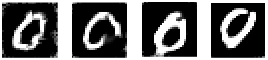}
    \end{subfigure}\begin{subfigure}{0.24\textwidth}
        \centering
        \includegraphics[width=1.\textwidth]{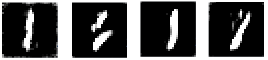}
    \end{subfigure}
    
    %% 2 3
    \begin{subfigure}{0.24\textwidth}
        \centering
        \includegraphics[width=1.\textwidth]{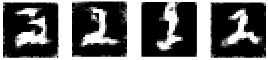}
    \end{subfigure}\begin{subfigure}{.24\textwidth}
        \centering
        \includegraphics[width=1.\textwidth]{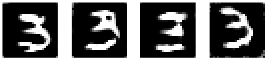}
    \end{subfigure}

    \caption{Examples from algorithm~\ref{alg:newtonraphson} with $T=64$ and $\eta=1$.}\label{fig:mnist_gen_main}
    %\vspace{-0.3cm}
    %Datasets: (a) One blob (b) Two circles (c) Two blobs (d) Two unbalanced blobs (e) Two moons.
\end{figure}

%\vspace{-0.1cm}
\subsubsection{Limitations}
%\vspace{-0.1cm}

We also tested our algorithm on the Cats versus Dogs dataset by training on Cats as the One Class and using Dogs as OOD examples. On this high-dimensional dataset, the AUROC barely exceeds $55.0\%$. This suggests that the SDF is relevant for tabular and simple image datasets (e.g MNIST) but fails to generalize in a meaningful way in higher dimensions. The distance used to define the learned SDF is the euclidian distance. It is well known that such distance has its limitations in pixel space. Using different distances that are more relevant to pixel space is a perspective left for future works.  Our constructions rely on our ability to solve Eikonal equation $\|\nabla_x f\|=1$. This is easy to enforce on dense layers, but it is still an active research area for convolutions (see the recent works of~\cite{achour2021existence,singlaimproved2022,xulot2022}).  

In addition, scores reported in Table~\ref{tab:mnistroc} use the same hyper-parameters for all the classes of Cifar-10, while some concurrent approaches~\cite{goyal_drocc_2020} tune them on a per-class basis. We observed significant improvements with per class tuning of the margin (see figure~\ref{tab:cheatingprotocolcifar} in appendix), but we chose not to report them in Table~\ref{tab:mnistroc}, because this practice is criticized by recent contributions on which we are aligned~\cite{han_adbench_nodate,dinghyperparameter2022}.
%A more thorough discussion of the limitations is left in Appendix~\ref{app:limitations}.

%\vspace{-0.1cm}
\section{OCSDF for implicit shape parametrization}
%\vspace{-0.1cm}
Our approach to learning the SDF contrasts with the computer graphics literature, where SDF is used to obtain the distance of a point to a surface (here defined as $\partial\support$). Indeed, SDFs are usually learned in a supervised fashion, requiring the ground truth of $l_2$ distance. This is classically achieved using Nearest-Neighbor algorithms, which can be cumbersome, especially when the number of points is high. Efficient data structures (e.g., using K-dtrees~\cite{maneewongvatana1999analysis} or Octrees~\cite{meagher1980octree}) can mitigate this effect but do not scale well to high dimensions. Instead, OCSDF learns the SDF solely based on points contained in the support.  While neural network training is not cheap by any mean, the network approach is advantageous at inference time. Indeed, with a dataset of size $n$, a single forward in the network costs $\BigO(1)$ (furthermore orthogonalization of matrices can be done once for all), while kNN costs $\BigO(n^2)$ or $\BigO(n\log{n})$ (with trees).

\begin{figure}[!ht]
    \centering
    \includegraphics[width=1.\linewidth, trim=0 2.6cm 0 0, clip]{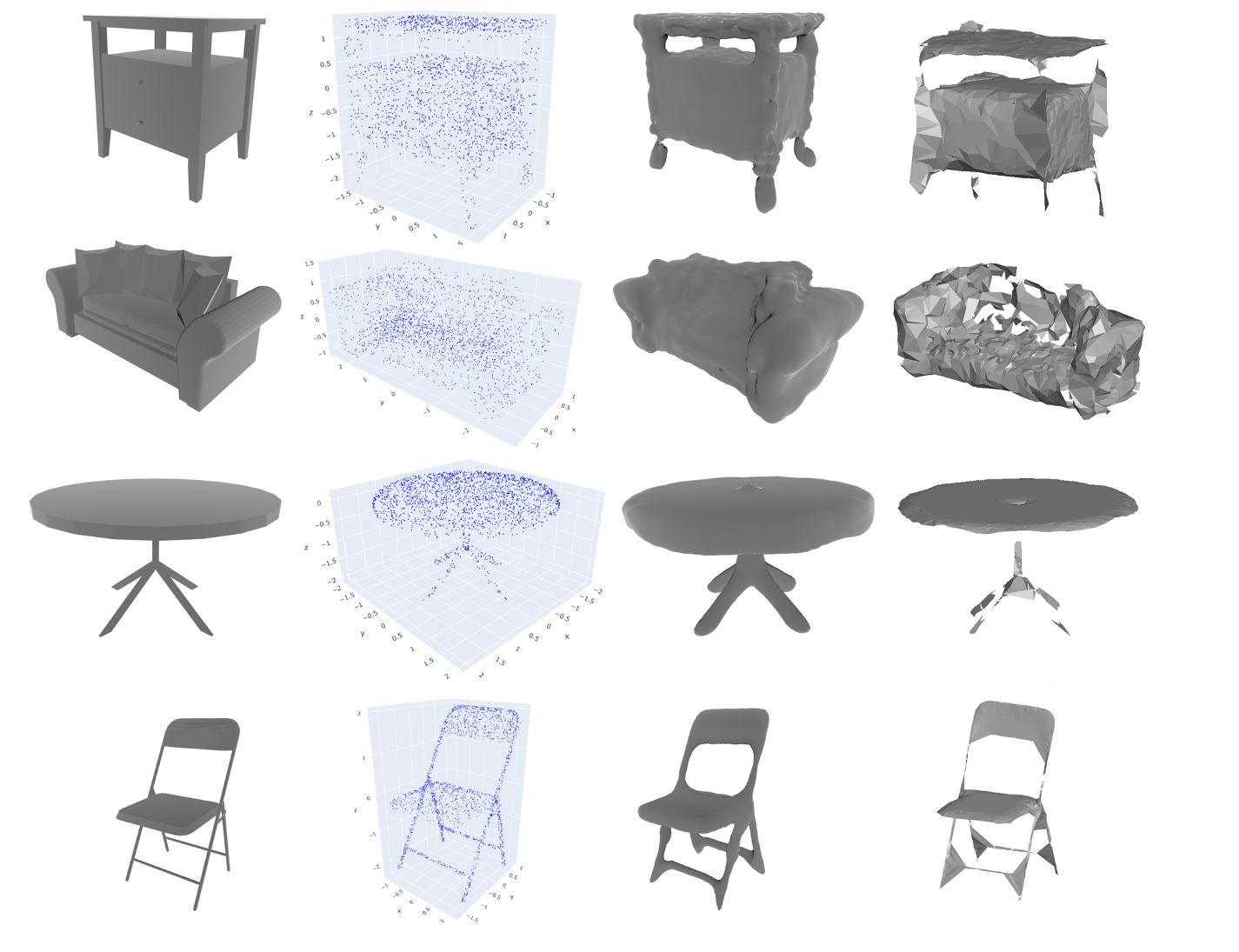}
    \caption{Visualization of the SDF (3rd column) from sparse point clouds of size $2048$ (2nd column) sampled from ground truth meshes (1st column) with \textit{Trimesh} library, against the SSSR algorithm~\cite{boltcheva2017surface} that attempts to reproduce the meshes solely from a point cloud (4th column). The SDF exhibits better extrapolation properties and provides smooth surfaces.}
    \label{fig:modelnet3d}
    %\vspace{-0.3cm}
\end{figure}

To illustrate this, we use models from {Princeton's ModelNet10 dataset~\cite{wu20153d}}. We sample $n=2048$ within each mesh to obtain a 3d point cloud. We fit the SDF on the point cloud with the same hyperparameters as the tabular experiment. We use Lewiner marching algorithm~\cite{lewiner2003efficient} from scikit-image~\cite{van2014scikit}, on a $200\times 200\times 200$ voxelization of the input. We plot the mesh reconstructed with Trimesh~\cite{trimesh}. The results are highlighted in figure~\ref{fig:modelnet3d}. We chose the first percentile of $\Expect_{x\sim \Prob_X}[f(x)]$ as the level set of the iso-surface we plot. We compare our results visually against a baseline from~\cite{boltcheva2017surface} implemented in Graphite~\cite{graphite} that rebuilds the mesh solely from the point cloud (without extrapolation). We highlight that $n=2048$ is considered low resolution; hence many details are expected to be lost. Nonetheless, our method recovers the global aspect of the shape more faithfully, smoothly, and consistently than the other baseline. 

%\vspace{-0.1cm}

%In table~\ref{tab:comparisondeepsdf} we highlight some key differences of our approach with the seminal work of~\cite{park2019deepsdf} on SDF in computer graphics. This table emphasizes that our method can be applied more 

%\vspace{-0.2cm}
\section{Conclusion and future work}
%\vspace{-0.1cm}
We proposed a new approach to One-Class Classification, OCSDF, based on the parametrization of the Signed Distance Function to the boundary of the known distribution using a 1-Lipschitz neural net. OCSDF comes with robustness guarantees and naturally tackles the lack of negative data inherent to OCC. Finally, this new method extends OCC beyond Out-of-Distribution detection and allows applying it to surface parametrization from point clouds and generative visualization.  

OCSDF relies on the appropriate sampling of a complementary distribution. It can be improved by integrating proper priors (e.g steps in Fourrier space, Perlin noise prior or negative data augmentation), rather than random uniform, as long as Definition~\ref{def:beddinformal} is fulfilled. This perspective is very relevant and will be investigated in future works.

\section{Acknowledgements}
This work has benefited from the AI Interdisciplinary Institute ANITI, which is funded by the French ``Investing for the Future – PIA3'' program under the Grant agreement ANR-19-P3IA-0004. The authors gratefully acknowledge the support of the DEEL project.\footnote{\url{https://www.deel.ai/}}  

%%%%%%%%% REFERENCES
{\small
\bibliographystyle{plainnat}
\bibliography{biblio}
}

\newpage
\onecolumn
\appendix

\tableofcontents

\section{Proofs and comments}\label{app:pandc}

\subsection{Complementary distribution}

\begin{definition}[$\bedd$ Complementary Distribution]\label{def:bedd}
Let $\Prob_X$ a distribution with compact support $\support \subset B$, with $B\subset\Reals^d$ a bounded measurable set. $Q$ is said to be $(B,\epsilon)$ disjoint from $\Prob_X$ if (i) its support $\supp Q\subset B$ is compact (ii) $d(\supp Q, \support)\geq 2\epsilon$ (iii) for all measurable sets $M\subset B$ such that $d(M,\support)\geq 2\epsilon$ we have $Q(M)>0$.  
It defines a symmetric but irreflexive binary relation denoted $Q\bedd\Prob_X$.  
\end{definition}

The idea is to learn one class classifier by reformulating one class learning of $\Prob_X$ as a binary classification of $\Prob_X$ against a carefully chosen adversarial distribution $Q(\Prob_X)$. This simple idea had already occurred repeatedly in the related literature~\cite{sabokrou2018adversarially}. Note that $\hkr$ benefits from generalization guarantees as proved in~\cite{bethune2022pay}: the optimal classifier on the train set and on the test set are the same in the limit of big samples.  
 
\subsection{Learning the SDF with HKR loss}

\hkrsdf*
    \begin{proof}
    The results follow from the properties of $\hkr$ loss given in Proposition 2  of~\cite{serrurier2021achieving}. If $Q\bedd \Prob_X$, then by definition, the two datasets are $2\epsilon$ separated. Consequently the hinge part of the loss is null: $\max{(0,m-yf(x))}$ for all pairs $(x,+1)$ and $(z,-1)$ with $x\sim\Prob_X$ and $z\sim Q$. We deduce that:
    \begin{equation}\label{eq:hkrsdfeq1}
        \forall x\in\support, f(x)\geq m,\qquad \forall z\in\supp Q, f(z)\leq -m.
    \end{equation}
    In the following we use:
    $$F_z=\{\argmin_{z_0\in\supp Q}\|x-z_0\|_2\},\forall x\in\support$$
    and
    $$F_x=\{\argmin_{x_0\in\support}\|x_0-z\|_2\},\forall z\in(\supp Q).$$
    Since $m=\epsilon$, we must have $f(x)=m$ for all $x\in F_x$, and $f(z)=-m$ for all $z\in F_z$, whereas $\Sdf(x)=0$ and $\Sdf(z)=-2m$. Thanks to the 1-Lipschitz property for every $x\in\support$ we have $f(x)\leq f(\partial x)+\|x-\partial x\|$ where $\partial x=\argmin_{\bar x\in\partial\support}\|x-\bar x\|$ is the projection of $x$ onto the boundary $\partial\support$. Similarly $f(z)\geq f(\partial z)-\|z-\partial z\|$. The $-yf(x)$ term in the $\hkr$ loss (Wasserstein regularization), incentives to maximize the amplitude $|f(x)|$ so the inequalities are tight.  Notice that $\Sdf(x)=\Sdf(\partial x)+\|x-\partial x\|$ and $\Sdf(z)=\Sdf(\partial z)-\|z-\partial z\|$. This allows concluding:
    \begin{equation}
        \forall x\in \support,\Sdf(x))= f^{*}(x)-m,\qquad \forall z\in\supp Q,\Sdf(z)= f^{*}(z)-m.
    \end{equation}

    \end{proof}
  
\subsection{Finding the right complementary distribution}

Finding the right distribution $Q\bedd\Prob_X$ with small $\epsilon$ is also challenging.

We propose to seek $Q$ through an alternating optimization process. Consider a sample $z\in\Level_t$. If $z$ is a \textit{false positive} (i.e $f_t(z)>0$ and $z\notin\support$), training $f_{t+1}$ on the pair $(z,-1)$ will incentive $f_{t+1}$ to fulfill $f_{t+1}(z)<0$, which will reduce the volume of false positive associated to $f_{t+1}$. If $z$ is a \textit{true negative} (i.e $f_t(z)<0$ and $z\notin\support$) it already exhibits the wanted properties. The case of \textit{false negative} (i.e $f_t(z)<0$ and $z\in\support$) is more tricky: the density of $\Prob_X$ around $z$ will play an important role to ensure that $f_{t+1}(z)>0$.  
  
Hence we assume that samples from the target $\Prob_X$ are \textit{significantly} more frequent than the ones obtained from pure randomness. It is a very reasonable assumption (especially for images, for example), and most distributions from real use cases fall under this setting.   

\begin{assumption}[$\Prob_X$ samples are more frequent than pure randomness (informal)]
For any measurable set $M\subset\support$ we have $\Prob_X(M)\gg \Uniform(M)$.
\end{assumption}
  
To ensure that property (iii) of definition~\ref{def:bedd} is fulfilled, we also introduce stochasticity in algorithm~\ref{alg:newtonraphson} by picking a random ``learning rate'' $\eta\sim\Uniform([0,1])$. To decorrelate samples, the learning rate is sampled independently for each example in the batch.  

The final procedure depicted in algorithm~\ref{alg:mmalgo} benefits from the mild guarantee of proposition~\ref{thm:fixpointstop}. It guarantees that once the complementary distribution has been found, the algorithm will continue to produce a sequence of complementary distributions and a sequence of classifiers $f_t$ that approximates $\Sdf$.  

\section{Signed Distance Function learning framed as Adversarial Training}\label{app:sdfadv}

\begin{property}\label{thm:fixpointstop}
Let $Q^{t}$ be such that $Q^{t}\bedd\Prob_X$. Assume that $Q^{t+1}$ is obtained with algorithm~\ref{alg:mmalgo}. Then we have $Q^{t+1}\bedd\Prob_X$. 
\end{property}
\begin{proof}
The proof also follows from the properties of $\hkr$ loss given in Proposition 2 of~\cite{serrurier2021achieving}
Since $Q_t\bedd\Prob_X$ all examples $z\sim Q_t$ generated fulfill (by definition) $d(z,\support)\geq 2\epsilon\geq 2m$. Indeed the 1-Lipschitz constraint (in property~\ref{prop:gnphkr}) guarantees that no example $z_t$ can ``overshoot'' the boundary. Hence for the associated minimizer $f_{t+1}$ of $\hkr$ loss, the hinge part of the loss is null. This guarantees that $f_{t+1}(z)\leq -m$ for $z\sim Q$. We see that by applying algorithm~\ref{alg:newtonraphson} the property is preserved: for all $z\sim Q_{t+1}$ we must have $f_{t+1}(z)\leq -m=-\epsilon$. Finally notice that because $z_0\sim\Uniform(B)$ and $\eta\sim\Uniform([0,1])$ the support $\supp Q$ covers the whole space $B$ (except the points that are less than $2\epsilon$ apart from $\support$). Hence we have $Q_{t+1}\bedd\Prob_X$ as expected.
\end{proof}

\begin{algorithm}%[tb]
\caption{Alternating Minimization for Signed Distance Function learning}
\label{alg:mmalgo}
\textbf{Input}: 1-Lipschitz neural network architecture $f_{\circ}$\\
\textbf{Input}: initial weights $\theta_0$, learning rate $\alpha$
\begin{algorithmic}[1] %[1] enables line numbers
\REPEAT
\STATE $f_t\gets f_{\theta_t}$
\STATE $\tilde\theta\gets \theta_t$
    \REPEAT
    \STATE Generate batch $z\sim Q_t$ of negative samples with algorithm~\ref{alg:newtonraphson}
    \STATE Sample batch $x\sim\Prob_X$ of positive samples
    \STATE Compute loss on batch $\Loss(\tilde\theta)\defeq\Empirical^{\text{hkr}}(f_{\tilde\theta},x,z)$
    \STATE Learning step $\tilde\theta\gets \tilde\theta+\alpha\nabla_{\theta}\Loss(\tilde\theta)$
    \UNTIL{convergence of $\tilde\theta$ to $\theta_{t+1}$.}
\UNTIL{convergence of $f_t$ to limit $f^{*}$.}
\end{algorithmic}
\end{algorithm}

\subsection{Lazy variant of algorithm~\ref{alg:mmalgo}}

Algorithm~\ref{alg:mmalgo} solves a MaxMin problem with alternating maximization-minimization. The inner minimization step (minimization of the loss over 1-Lipschitz function space) can be expensive. Instead, partial minimization can be performed by doing only a predefined number of gradient steps. This yields the lazy approach of algorithm~\ref{alg:mmalgolazy}. This provides a considerable speed up over the initial implementation. Moreover, this approach is frequently found in literature, for example, with GAN~\cite{} or WGAN~\cite{}. However, we lose some of the mild guarantees, such as the one of Proposition~\ref{thm:fixpointstop} or even Theorem~\ref{thm:hkrsdf}. Crucially, it can introduce unwanted oscillations in the training phase that can impede performance and speed. Hence this trick should be used sparingly.   

\begin{algorithm}%[tb]
\caption{One Class Signed Distance Function learning}
\label{alg:mmalgolazy}
\textbf{Input}: 1-Lipschitz neural net architecture $f_{\circ}$, initial weights $\theta_0$, learning rate $\alpha$, number of parameter update per time step $K$
\begin{algorithmic}[1] %[1] enables line numbers
\REPEAT
\STATE $\Tilde{\theta} \gets \theta_t$
\STATE Generate batch $z\sim Q_t$ of negative samples with algorithm~\ref{alg:newtonraphson}
\STATE Sample batch $x\sim\Prob_X$ of positive samples
\FOR{$K$ updates }
\STATE Compute loss on batch $\Loss(\theta)\defeq\Empirical^{\text{hkr}}(f_{\Tilde{\theta}},x,z)$
\STATE Learning step $\Tilde{\theta} \gets\Tilde{\theta}+\alpha\nabla_{\theta}\Loss(\Tilde{\theta})$
\ENDFOR
\STATE $\theta_{t+1} \gets \Tilde{\theta}$
\UNTIL{convergence of $f_t$.}
\end{algorithmic}
\end{algorithm}
\vspace{-0.3cm}

The procedure of algorithm~\ref{alg:mmalgolazy} bears numerous similarities with the adversarial training of Madry~\cite{madry2017towards}. In our case the adversarial examples are obtained by starting from noise $\Uniform(B)$ and relabeled are negative examples. In their case, the adversarial examples are obtained by starting from $\Prob_X$ itself and relabeled as positive examples.  

\section{Certifiable AUROC (Proposition \ref{certifauroc})}\label{app:certifauroc}

\begin{proposition}[certifiable AUROC]
 Let $F_0$ be the cumulative distribution function associated with the negative classifier's prediction (when $f(x) = 0$), and $p_1$ the probability density function of the positive classifier's prediction (when $f(x) = 0$). Then, for any attack of radius $\epsilon > 0$, the AUROC of the attacked classifier $f_{\epsilon}$ can be bounded by
\begin{equation}
    \text{AUROC}(f_{\epsilon}) = \int_{-\infty}^{\infty} F_0(t)p_1(t-2\epsilon)dt.
\end{equation}
\end{proposition}

\begin{proof}

Let $p_1$ (resp. $p_{-1}$) be the probability density function (PDF) associated with the classifier's positive (resp. negative) predictions. More precisely, $p_1$ (resp. $p_{-1}$) is the PDF of $f_{\sharp}\Prob_X$ (resp. $f_{\sharp}Q$) for some adversarial distribution $Q$, where $f_{\sharp}\cdot$ denotes the pushforward measure operator~\cite{bogachev2007measure} defined by the classifier. The operator $f_{\sharp}$ formalizes the shift between $\mathbb{P}_X$ (resp. $Q$), the ground truth distributions, and $p_{-1}$ (resp. $p_1$), the imperfectly distribution fitted by $f$. Let $F_{-1}$ and $F_1$ be the associated cumulative distribution functions. For a given classification decision threshold $\tau$, we can define the True Positive Rate (TPR) $F_{-1}(\tau)$, the True Negative Rate (TNR) $1 - F_1(\tau)$, and the False Positive Rate (FPR) $F_{1}(\tau)$. The ROC curve is then the plot of $F_{-1}(\tau)$ against $F_{1}(\tau)$. Hence, setting $v = F_1(t)$, we can define the AUROC as:
\begin{equation}
%https://stats.stackexchange.com/questions/180638/how-to-derive-the-probabilistic-interpretation-of-the-auc
    \text{AUROC}(f) = \int_{0}^{1} F_{-1}(F_1^{-1}(v))dv\\
\end{equation}

And with the change of variable $dv = p_1(t)dt$

\begin{equation}
    \text{AUROC}(f) = \int_{-\infty}^{\infty} F_{-1}(t)p_1(t)dt.
\end{equation}

We consider a scenario with symmetric attacks: the attack decreases (resp. increases) the normality score of One Class (resp. Out Of Distribution samples) for decision threshold $\tau\in\Reals$. When the 1-Lipschitz classifier $f$ is under attacks of radius at most $\epsilon>0$ we note $f_{\epsilon}$ the perturbed classifier:
\begin{equation}
    f_{\epsilon}(x)=\min_{\delta \leq \epsilon} (2\indicator\{f(x)\geq \tau\}-1) f(x + \delta).
\end{equation}
Note that $f_{\epsilon}(x)\leq f(x)+\epsilon$ when $f(x)<\tau$ and $f_{\epsilon}(x)\geq f(x)-\epsilon$ when $f(x)\geq \tau$ thanks to the 1-Lipschitz property. This effectively translates the pdf of $f_{\epsilon}$ by $|\epsilon|$ atmost.  
  
We obtain a lower bound for the AUROC (i.e a certificate):
\begin{equation}
\begin{aligned}
    %\text{AUROC}(f_{\epsilon}) &\geq \int_{-\infty}^{\infty} F_{-1}(F_1^{-1}(t-\epsilon\indicator_{t\geq \tau})+\epsilon\indicator_{t< \tau})dt\\
                               %&\geq \int_{-\infty}^{\infty} F_{-1}(F_1^{-1}(t-\epsilon)+\epsilon)dt\\
                               %&= \int_{-\infty}^{\infty} F_{-1}(t)p_1(t-2\epsilon)dt.
    \text{AUROC}(f_{\epsilon}) &\geq \int_{-\infty}^{\infty} F_{-1}(t+\epsilon)p_1(t-\epsilon)dt\\
                               &= \int_{-\infty}^{\infty} F_{-1}(t)p_1(t-2\epsilon)dt.
\end{aligned}
\end{equation}

\end{proof}
The certified AUROC score can be computed analytically without performing the attacks empirically, solely from score predictions $f_1(\tau-2\epsilon)$. More importantly, the certificates hold against \textit{any} adversarial attack whose $l2$-norm is bounded by $\epsilon$, regardless of the algorithm used to perform such attacks. We emphasize that producing certificates is more challenging than traditional defence mechanisms (e.g, adversarial training, see~\cite{ijcai2021p591} and references therein) since they do not target defence against a specific attack method.

% \begin{equation}
% %https://stats.stackexchange.com/questions/180638/how-to-derive-the-probabilistic-interpretation-of-the-auc
% \begin{aligned}
%     \text{AUROC}(f) &= \int_{-\infty}^{\infty} F_0(F_1^{-1}(t))dt\\
%                   &= \int_{-\infty}^{\infty} F_0(t)f_1(t)dt.
% \end{aligned}
% \end{equation}

% When the 1-Lipschitz classifier $f$ is under attacks of radius at most $\epsilon>0$ we note $f_{\epsilon}$ the perturbed classifier:
% \begin{equation}
%     f_{\epsilon}(x)=\argmin_{\delta} (2\indicator\{f(x)\geq \tau\}-1) f(x + \delta).
% \end{equation}

% We consider a scenario with symmetric attacks: the attack decreases (resp. increases) the normality score of One Class (resp. Out Of Distribution samples) for decision threshold $\tau\in\Reals$. We obtain a lower bound for the AUROC (i.e a certificate):
% \begin{equation}
% \begin{aligned}
%     \text{AUROC}(f_{\epsilon}) &\geq \int_{-\infty}^{\infty} F_0(F_1^{-1}(\tau-\epsilon\indicator_{t\geq T})+\epsilon\indicator_{\tau< T})dt\\
%                               &\geq \int_{-\infty}^{\infty} F_0(F_1^{-1}(t-\epsilon)+\epsilon)dt\\
%                               &= \int_{-\infty}^{\infty} F_0(t)p_1(t-2\epsilon)dt.
% \end{aligned}
% \end{equation}
\section{1-Lipschitz Neural Networks}\label{app:ortho}

We ensure that the kernel remains 1-Lipschitz by re-parametrizing them: $\Theta_i=\Pi(W_i)$ where $W_i$ is a set of unconstrained weights and $\Theta_i$ the orthogonal projection of $W_i$ on the Stiefel manifold - i.e the set of orthogonal matrices. The projection $\Pi$ is made with Bj\"orck algorithm~\cite{bjorck1971iterative}, which is differentiable and be included in the computation graph during forward pass. Unconstrained optimization is performed on $W_i$ directly.     

\section{Toy experiment in 2D}\label{app:toy2d}

All datasets are normalized to have zero mean and unit variance across all dimensions. The domain $B$ is chosen to be the ball of radius $5$. This guarantees that $\support\subset B$ for all datasets. The plots of figure\ref{fig:toy2d} are squares of sizes $[-5,5]$ for \textbf{(a)}, $[-3,3]$ for \textbf{(b)(c)(e)} and $[-4,4]$ for \textbf{(d)} to make the figure more appealing.     
  
\begin{figure}[!ht]
    \centering
    \includegraphics[width=\linewidth]{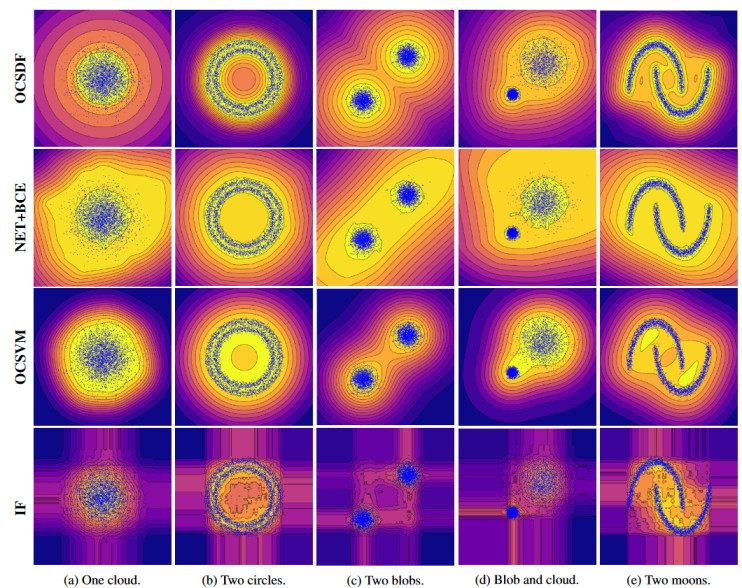}
    \caption{Toy examples of Scikit-learn. \textbf{Top row}: our method with Lipschitz (LIP) 1-Lipschitz network and $\hkr$ (HKR) loss. \textbf{Second row}: conventional network (NET) trained with Binary Cross Entropy (BCE). \textbf{Third row}: One Class SVM. \textbf{Fourth row}: Isolation Forest. }\label{fig:toy2dlarge}
    %Datasets: (a) One blob (b) Two circles (c) Two blobs (d) Two unbalanced blobs (e) Two moons.
\end{figure}

\subsection{One Class Learning}

All the toy experiments of figure~\ref{fig:toy2d} uses a $2\veryshortarrow512\veryshortarrow 512\veryshortarrow 512\veryshortarrow 512\veryshortarrow 1$ neural network. All the squares matrices are constrained to be orthogonal. The last layer is a unit norm column vector. The first layer consists of two unit norm columns that are orthogonal to each other. The optimizer is Rmsprop with default hyperparameters. The batch size is $b=256$, and the number of steps $T=4$ is small. We chose a margin $m=0.05$ except for ``blob and cloud'' dataset where we used $m=0.1$ instead. We take $\lambda=100$. The networks are trained for a total of $10,000$ gradient steps.    

\subsection{Other benchmarks}
  
For One Class SVM, we chose a parameter $\nu=0.05$, $\gamma=\frac{1}{2}$ (which corresponds to the ``scale'' behavior of scikit-learn for features in 2D with unit variance) and the popular RBF kernel. For Isolation Forest, we chose a default contamination level of $0.05$.   
  
The conventional network (without orthogonality constraint) is trained with Binary Cross Entropy (also called log-loss) and Adam optimizer. It shares the same architecture with ReLU instead of GroupSort. It is also trained for a total of $10,000$ gradient steps for a fair comparison. It would not make sense to use $\hkr$ loss for a conventional network since it diverges during training, as noticed in~\cite{bethune2022pay}. The Lipschitz constant of the conventional network grows uncontrollably during training even with Binary Cross-Entropy loss, which is also compliant with the results of~\cite{bethune2022pay}.

\begin{figure}[!ht]
    \centering
    \includegraphics[width=\linewidth]{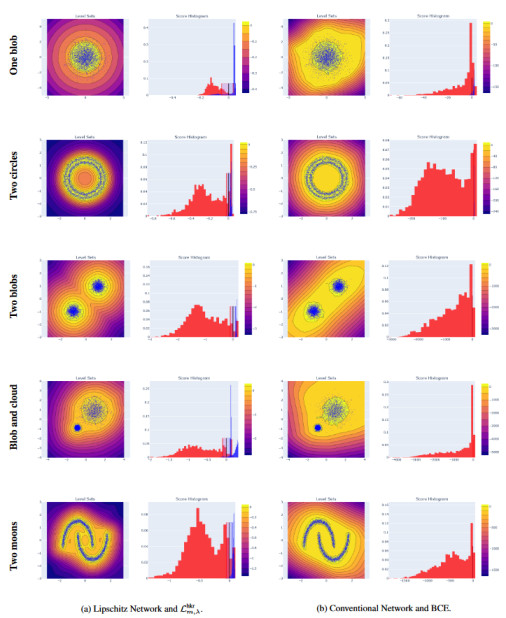}
    \caption{Histograms of score functions for 1-Lipschitz network (left) and conventional neural network. The blue bars correspond to the distribution of $f(x),x\sim\Prob_X$ and the red bars to the distribution $f(z),z\sim\Uniform(B)$.}\label{fig:toy2dhist}
\end{figure}

\section{Toy experiments on 3D point clouds}\label{app:toy3d}

In figures~\ref{fig:cadtool1} and~\ref{fig:cadtool2}, we provide additional examples of the use of SDF to reconstruct shapes from point clouds. We also compare OCSDF with Deep SDF \cite{park2019deepsdf}, a standard baseline for neural implicit surface parametrization, to highlight the practical advantages of our method.

\begin{table}[ht]
    \centering
    %\hspace{-0.35cm}
    \small
    \begin{tabular}{c|cc}
    \hline
      \makecell{Methods\\} & \makecell{OCSDF\\ (ours)} & \makecell{DeepSDF\\ \cite{park2019deepsdf}}\\
    \hline
    %\hline
    Target & \makecell{Generate $Q\bedd P$} & \makecell{Compute $\Sdf(x)$ from $P$}\\
    Cost         & \makecell{Backward pass} & \makecell{Nearest Neighbor search} \\
    Loss         & HKR $\hkr$ & $\|f(x)-\Sdf(x)\|^2_2$\\
    Guarantees   & $f$ is 1-Lipschitz & None\\
    \end{tabular}
    \caption{Comparison of our approach against DeepSDF.}
    \label{tab:comparisondeepsdf}
\end{table}

\begin{figure}[!ht]
    \centering
    \includegraphics[width=1.\linewidth]{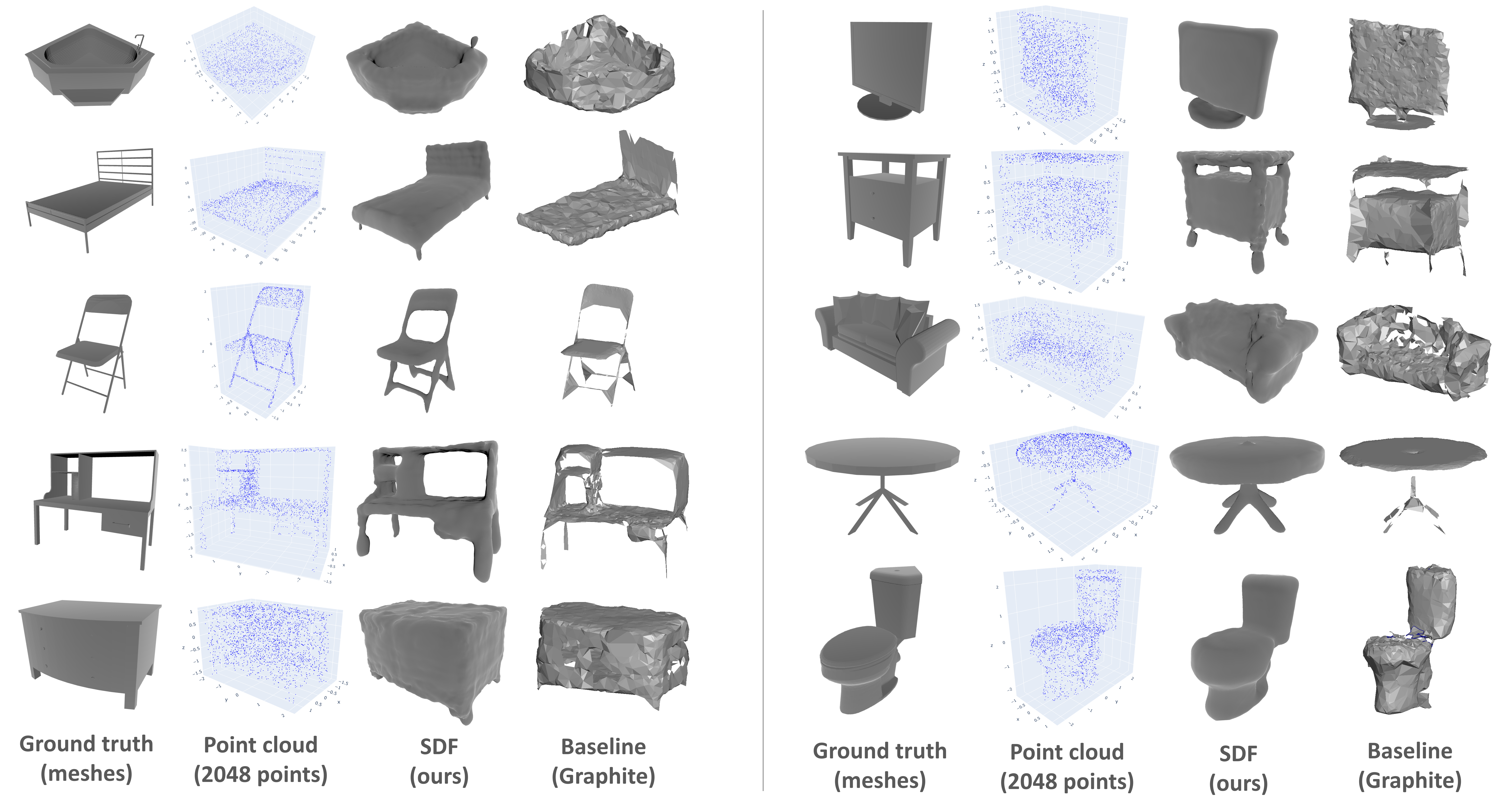}
    \caption{SDF of $2048$ points sampled from the mesh of a 3D CAD model.}
    \label{fig:cadtool1}
\end{figure}

\begin{figure}[!ht]
    \centering
    \includegraphics[width=0.7\linewidth]{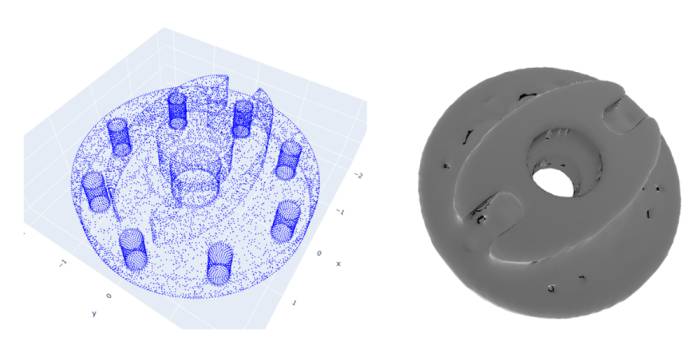}
    \caption{SDF of $2048$ points sampled from the mesh of a 3D CAD model.}
    \label{fig:cadtool2}
\end{figure}

\section{Tabular data}

The optimizer is RMSprop with default hyper-parameters. We use the lazy variant, i.e algorithm~\ref{alg:mmalgolazy}. We chose $\lambda=100$ and $m\in\{0.01, 0.05, 0.2, 1.\}$. The results are averaged over $20$ runs, and we report the highest average among all values of $m$. We use a batch size $b=128$ a number of steps $T=4$. The network is trained for a total of $40$ epochs over the one class, using a warm start of $5$ epoch with $T=0$.  

\paragraph{One Class (OC) protocol}: in this setting the normal examples are split in train and test set (approximately $50\%$ of data each). The network is trained on normal \textit{train} examples. The AUROC score is computed between  normal \textit{test} examples and anomalies. We report the results in table~\ref{tab:tabularrococ}.  

\paragraph{Anomaly Detection (AD) protocol}: in this setting the normal examples and anomalies are part of the train set. The network is trained on all examples. The AUROC score is computed between  normal examples and anomalies. We report the results in table~\ref{tab:tabularrocad}.  

\begin{table}%[]
    \centering
    \small
    %\hspace{-0.35cm}
    \begin{tabular}{ccc|c}
    \hline
      Dataset & $\bm{d}$ &\#train & OCSDF (Ours) \\
    \hline
    \hline
    Arrhythmia & 274 & 193 & $80.0\pm 1.3$\\
    Thyroid & 6 & 1,839 & $98.3\pm 0.1$\\     
    Mammography & 6 & 5,591 & $88.0\pm 1.4$\\
    Vowels & 12 & 703 & $96.1\pm 0.9$\\
    Satimage-2 & 36 & 2,866 & $97.8\pm 0.0$\\
    smtp & 3 & 47,553 & $83.8\pm 3.5$\\
    \end{tabular}
    \caption{AUROC score on \textit{test set + anomalies} with tabular data, averaged over $20$ runs. The dimension of the dataset is denoted by $\bm{d}$. In the \textbf{One Class protocol (OC)}, the classifier is trained on a subset of normal data, and test on unseen normal data and anomalous data. The ``\#train'' column indicates train set size.}
    \label{tab:tabularrococ}
\end{table}

\section{One Class Classification of image data}\label{app:occimage}

\subsection{Mnist experiment}

The MNIST images are normalized such that pixel intensity lies in $[-1,1]$ range. The set $B$ is chosen to be the image space, i.e $B=[-1,1]^{28\times 28}=[-1,1]^{784}$. The optimizer is RMSprop with default hyperparameters. We chose $m=0.02\times \sqrt{(28\times 28\times (1-(-1))}\approx 0.79$ : this corresponds to modification of $1$\% of the maximum norm of an image. We take $\lambda=200$. We use a batch size $b=128$, and a number of steps $T=16$. We use the lazy variant, i.e algorithm~\ref{alg:mmalgolazy}. The network is trained for a total of $70$ epochs over the one class (size of the support: $\approx 6000$ examples), using a warm start of $10$ epoch with $T=0$. The learning rate follows a linear decay from $1e^{-3}$ to $1e^{-6}$. 

All the experiments use a VGG-like architecture, depicted in table~\ref{tab:mnistarch}. Convolutions are parametrized using layers of Deel-Lip library~\cite{serrurier2021achieving}, which use an orthogonal kernel with a corrective constant on the upper bound of the Lipschitz constant of the kernel. Dense layers also use orthogonal kernels. We use \textit{l2-norm-pooling} layer with windows of size $2\times 2$ that operates as: $(x_{11},x_{12},x_{21},x_{22})\mapsto \|[x_{11},x_{12},x_{21},x_{22}]\|$.  
  
We see in table~\ref{tab:mnistroc} that it yields competitive results against other naive baselines such as Isolation Fortest (IF) or One Class SVM (OCSVM), and against the popular deep learning algorithm \textit{Deep SVDD}~\cite{ruff2018deep}.  

\paragraph{Against DeepSVDD.} We test 4 configurations for DeepSVDD.
\begin{enumerate}
    \item The one from the original's paper~\cite{ruff2018deep} with LeNet architecture.
    \item The Pytorch implementation from their official repository that combines LeNet and BatchNorm: \url{https://github.com/lukasruff/Deep-SVDD-PyTorch}.
    \item The Tensorflow implementation from pyod~\cite{zhao2019pyod} with dense layers and l$2$ activation regularization.
    \item The Tensorflow implementation from pyod but with the same (unconstrained) architecture as Table~\ref{tab:mnistarch} to ensure a fair comparison.
\end{enumerate}
Each configuration is run $10$ times, and the results are averaged. The best average among the four configurations is reported in the main paper, while we report in Table~\ref{tab:extensivedeepsvdd} the results of all the configurations. It appears that the method of the original paper is very sensitive to the network's architecture, and the results are hard to reproduce overall. Empirically, we also observe that the test/OOD AUROC is often smaller than train/OOD AUROC, which suggests that DeepSVDD memorizes the train set but do not generalize well outside of its train data. We also notice that those results are different from the ones originally reported by authors in their paper.   
  
\begin{table}%[]
    \centering
    \begin{tabular}{|c|cc|}
            \hline
            Datasets & Toy2D \& tabular\& Implicit Shape & Mnist \& Cifar-10 \\
            \hline
            \hline
            \# parameters & 792K  &  2,103K  \\
            \hline
            layer 1 &   dense-512 (fullsort)  &  conv-3x3-64 (groupsort)\\
            layer 2 &   dense-512 (fullsort)  &  conv-3x3-64 (groupsort)\\
            layer 3 &   dense-512 (fullsort)  &  conv-3x3-64 (groupsort)\\
            layer 4 &   dense-512 (fullsort)  &  conv-3x3-64 (groupsort)\\
            layer 5 &   dense-1 (linear)      &  l2 norm pooling - 2x2 \\
            layer 6 &                         &  conv-3x3-128 (groupsort)\\
            layer 7 &                         &  conv-3x3-128 (groupsort)\\
            layer 8 &                         &  conv-3x3-128 (groupsort)\\
            layer 9 &                         &  conv-3x3-128 (groupsort)\\
            layer 10 &                         &  l2 norm pooling - 2x2 \\
            layer 11 &                         &  conv-3x3-256 (groupsort)\\
            layer 12 &                         &  conv-3x3-256 (groupsort)\\
            layer 13 &                         &  conv-3x3-256 (groupsort)\\
            layer 14 &                         &  l2 norm pooling - 2x2 \\
            layer 15 &                         &  global l2 norm pooling \\
            layer 16 &                         &  dense-1 (linear) \\
            \hline
            \end{tabular}
    \caption{Architectures of the 1-Lipschitz networks used for the experiments, using layers of Deel-Lip library.}
    \label{tab:mnistarch}
\end{table}

\begin{table}%[]
    \centering
    \begin{tabular}{c|cccc}
            \hline
            Protocol & \makecell{As described\\ in article} & \makecell{As in Pytorch\\ repository} & \makecell{As in Pytorch repository\\ (best run)} & \makecell{Pyod\\ \cite{zhao2019pyod}}\\
            \hline
            \hline
            mean & $79.8\pm 2.9$ & $88.6\pm 6.1$ & $94.0\pm 3.6$ & $85.0\pm 7.8$\\
            \hline
            0 & $88.6\pm 0.$ & $90.2\pm 9.4$ & $97.2$ & $83.7\pm 0.$\\
            1 & $97.1\pm 0.$ & $99.6\pm 0.1$ & $99.7$ & $97.2\pm 0.$\\
            2 & $66.5\pm 0.$ & $79.9\pm 8.8$ & $88.8$ & $80.5\pm 0.$\\
            3 & $70.9\pm 0.$ & $80.5\pm 8.4$ & $91.3$ & $79.6\pm 0.$\\
            4 & $78.0\pm 0.$ & $88.3\pm 7.7$ & $93.9$ & $85.4\pm 0.$\\
            5 & $69.7\pm 0.$ & $82.8\pm 6.0$ & $89.5$ & $72.9\pm 0.$\\
            6 & $85.2\pm 0.$ & $95.8\pm 1.9$ & $97.8$ & $93.7\pm 0.$\\
            7 & $88.9\pm 0.$ & $91.0\pm 3.3$ & $94.7$ & $92.8\pm 0.$\\
            8 & $72.8\pm 0.$ & $84.6\pm 4.7$ & $90.5$ & $74.8\pm 0.$\\
            9 & $79.9\pm 0.$ & $93.3\pm 3.0$ & $96.1$ & $89.9\pm 0.$\\
            \end{tabular}
    \caption{Average AUROC over $10$ runs for each digit of Mnist. We re-code the method in \textit{Tensorflow} using various configurations, due to the discrepancies of protocols found in the literature. In overall, DeepSVDD is very sensitive to the network architecture and can dramatically overfit.}
    \label{tab:extensivedeepsvdd}
\end{table}
  
We report the histogram of predictions for the train set (One class), the set test (One class) and Out Of Distribution (OOD) examples (the classes of the test set) in figure~\ref{fig:mnisthisto}.

\begin{figure}[!ht]
    \centering
    \includegraphics[width=\linewidth]{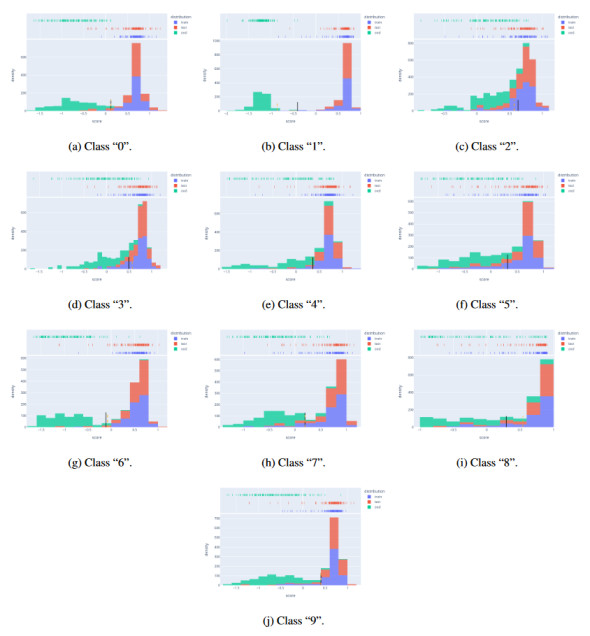}
    \caption{Histogram of scores predicted by the classifier at the end of training for \textbf{train} examples (in blue), \textbf{test} examples (in red), and \textbf{OOD Test} examples (in green) on Mnist after one of the runs of the algorithm.}\label{fig:mnisthisto}
\end{figure}

\begin{table}%[]
    \centering
    \begin{tabular}{|c|ccccc|}
        \hline
       Dataset          & One cloud  & Two circles & Two blobs & Blob and cloud & Two moons \\
       \hline
       LLC conventional &   26.66    &    122.84   &  1421.41  &      53.90     &   258.73  \\ 
       \hline
    \end{tabular}
    \caption{Lower bound on the Local Lipschitz Constant (LLC) of the conventional network after $10,000$ training steps. It is obtained by computing the maximum of $\|\nabla_{x_i}f(x_i)\|$ over the train set. The LLC of the conventional network grows uncontrollably during training and fails to provide metric guarantees.}
    \label{tab:llc}
\end{table}

\subsection{Cifar-10}

The Cifar-10 images are centedered-reduced per channel (the certificates and the attacks take into account the rescaling in the computation of the radius). All pixel intensities fall in $[-2.1,2.14]$ interval. The set $B$ is chosen to be the image space, i.e $B=[-2.1,2.14]^{32\times 32\times 3}=[-2.1, 2.14]^{3072}$. The optimizer is RMSprop with default hyperparameters. We chose $m=0.002\times \sqrt{(32\times 32\times (2.14-(-2.1))}\approx 0.11$ : this corresponds to modification of $0.2$\% of the maximum norm of an image. We take $\lambda=1000$. We use a batch size $b=128$, and a number of steps $T=32$. We use the lazy variant, i.e algorithm~\ref{alg:mmalgolazy}. The network is trained for a total of $90$ epochs over the one class (size of the support: $\approx 5000$ examples), using a warm start of $10$ epoch with $T=0$. The learning rate follows a linear decay from $2.5e^{-4}$ to $2.5e^{-7}$.  

Note that the \textit{same} hyper-parameters are used for all the runs and all the classes. This is different from some practices of litterature (see fig~11. in~\cite{goyal_drocc_2020} ). Unfortunately this protocol is somewhat unfair since it explicitly leverages prior knowledge on the other classes - it can not be used in a real world scenario where the other classes are not available. Nonetheless, for completeness, we report in table~\ref{tab:cheatingprotocolcifar} the clean AUROC on Cifar10 with \textit{per class} tuning of the margin $m$.

\begin{table*}%[]
    \centering
    %\hspace{-0.35cm}
    \small
    \begin{tabular}{|cccccccccc|c|}
    \hline
    Airplane & Automobile & Bird & Cat & Deer & Dog & Frog & Horse & Ship & Truck & Mean\\
    \hline
    \hline
    $79.0$ & $65.2$ & $55.4$ & $60.4$ & $52.5$ & $56.9$ & $57.7$ & $55.0$ & $71.9$ & $67.0$ & $62.1\pm 8.0$\\
    \hline
    \end{tabular}
    \caption{Best AUROC of OCSDF on Cifar-10 with \textit{per class} tuning of the margin $m$.}
    \label{tab:cheatingprotocolcifar}
\end{table*}

\subsection{Empirical robustness}

We also report the complete results of the empirical robustness of OCSDF versus DeepSVDD for MNIST and Cifar10 in tables \ref{tab:empiricalmnistsvdd}, \ref{tab:empiricalOCSDFmnist}, \ref{tab:empiricalcifar10svdd} and \ref{tab:empiricalOCSDFcifar10}.

\begin{table*}%[]
    \centering
    %\hspace{-0.35cm}
    \small
    \begin{tabular}{|c|cccccc|}
    \hline
      Mnist & DeepSVDD & DeepSVDD & DeepSVDD & DeepSVDD & DeepSVDD & DeepSVDD \\
      l$2$-PGD & $\epsilon=8/255$ & $\epsilon=16/255$ & $\epsilon=36/255$ & $\epsilon=72/255$ & $\epsilon=144/255$ & $\epsilon=255/255$\\
    \hline
    \hline
    mean & $90.2\pm 6.4$ & $89.7\pm 6.9$ & $88.2\pm 8.4$ & $84.5\pm 12.2$ & $74.1\pm 21.7$ & $61.9\pm 29.0$\\
    \hline
    digit 0 & $96.9$ & $96.5$ & $95.3$ & $91.6$ & $78.8$ & $56.9$\\
    digit 1 & $99.6$ & $99.6$ & $99.6$ & $99.5$ & $99.5$ & $99.4$\\
    digit 2 & $80.6$ & $78.9$ & $73.7$ & $61.2$ & $29.4$ & $08.4$\\
    digit 3 & $87.2$ & $86.7$ & $85.4$ & $82.2$ & $69.6$ & $47.8$\\
    digit 4 & $90.7$ & $90.7$ & $90.7$ & $90.5$ & $90.0$ & $89.4$\\
    digit 5 & $83.0$ & $81.5$ & $76.7$ & $65.9$ & $44.2$ & $28.8$\\
    digit 6 & $97.7$ & $97.6$ & $97.0$ & $95.2$ & $86.8$ & $69.2$\\
    digit 7 & $93.3$ & $93.2$ & $92.9$ & $92.4$ & $90.9$ & $89.4$\\
    digit 8 & $82.7$ & $82.1$ & $80.3$ & $76.1$ & $61.5$ & $39.7$\\
    digit 9 & $90.5$ & $90.5$ & $90.4$ & $90.4$ & $90.2$ & $89.7$\\
    \hline
    \end{tabular}
    \caption{Empirical AUROC against l$2$-PGD adversarial attacks on Mnist for DeepSVDD.}
    \label{tab:empiricalmnistsvdd}
\end{table*}

\begin{table*}%[]
    \centering
    %\hspace{-0.35cm}
    \small
    \begin{tabular}{|c|cccccc|}
    \hline
      Mnist & OCSDF & OCSDF & OCSDF & OCSDF & OCSDF & OCSDF \\
      l$2$-PGD & $\epsilon=8/255$ & $\epsilon=16/255$ & $\epsilon=36/255$ & $\epsilon=72/255$ & $\epsilon=144/255$ & $\epsilon=255/255$\\
    \hline
    \hline
    mean & $93.2\pm 5.0$ & $92.8\pm 5.3$ & $91.8\pm 6.0$ & $89.8\pm 7.4$ & $83.1\pm 11.9$ & $68.2\pm 19.5$\\
    \hline
    digit 0 & $99.4$ & $99.4$ & $99.3$ & $99.1$ & $98.2$ & $94.9$\\
    digit 1 & $99.1$ & $99.0$ & $98.8$ & $98.4$ & $96.7$ & $90.9$\\
    digit 2 & $89.7$ & $89.1$ & $87.5$ & $84.4$ & $74.0$ & $51.0$\\
    digit 3 & $91.1$ & $90.5$ & $89.2$ & $86.4$ & $77.3$ & $56.6$\\
    digit 4 & $95.2$ & $94.9$ & $94.1$ & $92.6$ & $87.1$ & $73.0$\\
    digit 5 & $84.6$ & $83.7$ & $81.2$ & $76.3$ & $60.8$ & $33.3$\\
    digit 6 & $98.0$ & $97.9$ & $97.6$ & $96.9$ & $94.4$ & $86.9$\\
    digit 7 & $95.3$ & $95.0$ & $94.4$ & $93.0$ & $88.5$ & $77.2$\\
    digit 8 & $85.8$ & $85.1$ & $83.3$ & $79.7$ & $68.7$ & $46.9$\\
    digit 9 & $93.9$ & $93.6$ & $92.7$ & $91.0$ & $85.2$ & $70.9$\\
    \hline
    \end{tabular}
    \caption{Empirical AUROC against l$2$-PGD adversarial attacks on Mnist for one run of OCSDF in One Class learning.}
    \label{tab:empiricalOCSDFmnist}
\end{table*}

\begin{table*}%[]
    \centering
    %\hspace{-0.35cm}
    \small
    \begin{tabular}{|c|cccccc|}
    \hline
      CIFAR10 & DeepSVDD & DeepSVDD & DeepSVDD & DeepSVDD & DeepSVDD & DeepSVDD \\
      l$2$-PGD & $\epsilon=8/255$ & $\epsilon=16/255$ & $\epsilon=36/255$ & $\epsilon=72/255$ & $\epsilon=144/255$ & $\epsilon=255/255$\\
    \hline
    \hline
    mean & $54.4\pm 8.3$ & $46.5\pm 8.1$ & $28.4\pm 7.0$ & $09.2\pm 3.7$ & $0.7\pm 0.6$ & $0.1\pm 0.0$\\
    \hline
    Airplane & $62.82$ & $55.09$ & $36.21$ & $12.57$ & $0.76$ & $0.01$\\
    Automobile & $43.02$ & $36.94$ & $21.98$ & $8.35$ & $1.24$ & $0.09$\\
    Bird & $58.50$ & $49.13$ & $27.66$ & $6.47$ & $0.16$ & $0.01$\\
    Cat & $45.80$ & $36.82$ & $18.76$ & $3.91$ & $0.12$ & $0.00$\\
    Deer & $63.02$ & $53.55$ & $30.82$ & $7.41$ & $0.23$ & $0.00$\\
    Dog & $50.33$ & $43.32$ & $27.95$ & $10.88$ & $1.43$ & $0.09$\\
    Frog & $63.44$ & $55.31$ & $35.33$ & $11.12$ & $0.63$ & $0.01$\\
    Horse & $44.67$ & $36.55$ & $20.09$ & $5.37$ & $0.26$ & $0.00$\\
    Ship & $64.38$ & $57.67$ & $40.89$ & $17.47$ & $1.87$ & $0.04$\\
    Truck & $48.39$ & $40.85$ & $24.79$ & $8.46$ & $0.81$ & $0.03$\\
    \hline
    \end{tabular}
    \caption{Empirical AUROC against l$2$-PGD adversarial attacks on Cifar10 for DeepSVDD.}
    \label{tab:empiricalcifar10svdd}
\end{table*}

\begin{table*}%[]
    \centering
    %\hspace{-0.35cm}
    \small
    \begin{tabular}{|c|cccccc|}
    \hline
      CIFAR10 & OCSDF & OCSDF & OCSDF & OCSDF & OCSDF & OCSDF \\
      l$2$-PGD & $\epsilon=8/255$ & $\epsilon=16/255$ & $\epsilon=36/255$ & $\epsilon=72/255$ & $\epsilon=144/255$ & $\epsilon=255/255$\\
    \hline
    \hline
    mean & $59.0\pm 9.3$ & $58.2\pm 9.2$ & $56.7\pm 9.6$ & $53.9\pm 10.3$ & $48.4\pm 10.3$ & $39.8\pm 13.8$\\
    \hline
    Airplane & $80.1$ & $79.8$ & $79.0$ & $77.4$ & $73.9$ & $67.6$\\
    Automobile & $60.2$ & $59.5$ & $57.7$ & $54.4$ & $47.3$ & $35.7$\\
    Bird & $52.7$ & $51.9$ & $49.9$ & $46.2$ & $38.8$ & $28.0$\\
    Cat & $57.7$ & $57.0$ & $55.4$ & $52.3$ & $46.1$ & $36.6$\\
    Deer & $50.8$ & $50.0$ & $48.0$ & $44.4$ & $37.3$ & $27.1$\\
    Dog & $55.8$ & $55.2$ & $53.5$ & $50.5$ & $44.4$ & $35.0$\\
    Frog & $54.7$ & $54.0$ & $52.2$ & $49.1$ & $42.8$ & $32.9$\\
    Horse & $48.0$ & $47.3$ & $45.4$ & $42.0$ & $35.3$ & $25.5$\\
    Ship & $71.7$ & $69.2$ & $68.7$ & $67.5$ & $65.1$ & $61.5$\\
    Truck & $57.9$ & $57.6$ & $56.9$ & $55.4$ & $52.5$ & $48.1$\\
    \hline
    \end{tabular}
    \caption{Empirical AUROC against l$2$-PGD adversarial attacks on Cifar10 for one run of OCSDF in One Class learning.}
    \label{tab:empiricalOCSDFcifar10}
\end{table*}

\subsection{Image synthesis from One Class classifier}

We perform a total of $T=64$ steps with algorithm~\ref{alg:newtonraphson} by targeting the level set $f^{-1}(\{\Expect_{\Prob_X}[f(x)]\})$ (to ensure we stay inside the distribution). Moreover, during image synthesis, we follow a deterministic path by setting $\eta=1$. The images generated were picked randomly (with respect to initial sample $z_0\sim\Uniform(B)$) without cherry-picking. Interestingly, we see that the algorithm recovers some \textit{in-distribution} examples successfully.    
The examples for which the image is visually deceptive are somewhat correlated with low AUC scores. Those failure cases are also shared by concurrent methods, which suggests that some classes are harder to distinguish. Notice that sometimes Out Of Distribution (OOD) Mnist digits, from \textit{other classes not seen during training}, are sometimes generated. This suggests that most Mnist digits can be built from a small set of elementary features that are combined during the generation of $Q_t$. Visualizations of generated digits are presented in Figure \ref{fig:mnist_gen}.

\begin{figure}[!ht]
    \centering
    
    %% 0 1
    \begin{subfigure}{0.4\textwidth}
        \centering
        \includegraphics[width=1.\textwidth]{new_images/mnist/gan_0_cropped.PNG}
    \end{subfigure}\begin{subfigure}{0.4\textwidth}
        \centering
        \includegraphics[width=1.\textwidth]{new_images/mnist/gan_1_cropped.png}
    \end{subfigure}
    
    %% 2 3
    \begin{subfigure}{0.4\textwidth}
        \centering
        \includegraphics[width=1.\textwidth]{new_images/mnist/gan_2_cropped.png}
    \end{subfigure}\begin{subfigure}{0.4\textwidth}
        \centering
        \includegraphics[width=1.\textwidth]{new_images/mnist/gan_3_cropped.png}
    \end{subfigure}
    
    %% 4 5
    \begin{subfigure}{0.4\textwidth}
        \centering
        \includegraphics[width=1.\textwidth]{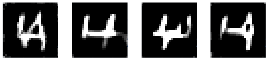}
    \end{subfigure}\begin{subfigure}{0.4\textwidth}
        \centering
        \includegraphics[width=1.\textwidth]{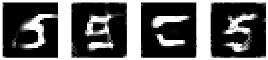}
    \end{subfigure}
    
    %% 6 7
    \begin{subfigure}{0.4\textwidth}
        \centering
        \includegraphics[width=1.\textwidth]{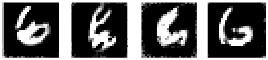}
    \end{subfigure}\begin{subfigure}{0.4\textwidth}
        \centering
        \includegraphics[width=1.\textwidth]{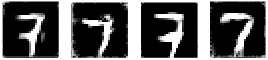}
    \end{subfigure}

    %% 8 9
    \begin{subfigure}{0.4\textwidth}
        \centering
        \includegraphics[width=1.\textwidth]{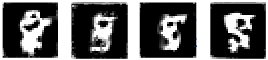}
    \end{subfigure}\begin{subfigure}{0.4\textwidth}
        \centering
        \includegraphics[width=1.\textwidth]{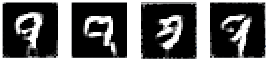}
    \end{subfigure}
    
    \caption{Synthetic examples obtained by running algorithm~\ref{alg:newtonraphson} with $T=64$ and $\eta=1$.}\label{fig:mnist_gen}
    %Datasets: (a) One blob (b) Two circles (c) Two blobs (d) Two unbalanced blobs (e) Two moons.
\end{figure}
  
Finally, note that our method is not tailored for example generation: this is merely a side effect of the training process of the classifier. There is no need a the encoder-decoder pair of VAE nor the discriminator-generator pair of a GAN. Moreover, no hyper-parameters other than $m$ and $T$ are required.   

\section{Known Limitations}\label{app:limitations}

Despite its performance and appealing properties, the method suffers from some important limitations we highlight below and that can serve as a basis for future work. 

\subsection{Gradient Norm Preserving networks (GNP) approximation power}

The performance of the algorithm strongly depends on its capacity to properly learn the true minimizer $f^{*}$ of $\hkr$ loss. Per Property~\ref{prop:gnphkr} such minimizer must fulfill $\|\nabla_x f^{*}(x)\|_2=1$ everywhere on the support of $\Prob_X$ and $Q_t$. Hence the performance of the algorithm (and the associated theoretical guarantees) depends on the capacity of the GNP network to fulfil this property. In the tabular case, it is easy to do using orthogonal matrices for affine layers and GroupSort (or FullSort~\cite{anil2019sorting}) activation functions. However, in the image case, designing ``orthogonal convolution'' is still an active research area. Several solutions have been proposed, but they come with various drawbacks in terms of simplicity of implementation, computational cost, or tightness of the constraint. Hence the average gradient norm on image datasets struggles to exceed $0.3$ in practice. Another limitation stems from low-rank adjoint operators (e.g the last layer of the network): during backpropagation they do not preserve gradient norm along all directions.  
  
The Newton-Raphson trick that uses steps of size $\frac{\nabla_x f(x)}{\|\nabla_x f(x)\|_2^2}$ mitigates partially the issue. This suggests that the algorithm (in its current form) could benefit from further progress in Gradient Norm Preserving (GNP) architectures.  

\subsection{Limitations of the euclidean norm in image space}

The algorithm provides metric guarantees in the construction of the Signed Distance Function (SDF) to the boundary. The $l2$-norm is not a crucial component of the construction: the proof of~\ref{thm:hkrsdf} and Proposition 2 of~\cite{serrurier2021achieving} can be applied to any norm. However, in every case, the Lipschitz constraint $|f(x)-f(y)|\leq \|x-y\|_L$ on the network architecture must coincide with the norm $\|\cdot\|_L$ used to build the signed distance function. Currently, only networks that are Lipschitz with respect to $l^{\infty}$ and $l^2$ norms benefit from universal approximation properties~\cite{anil2019sorting}. Those norms are often meaningful for tabular data, but not for images. Hence, metric guarantees are less useful in pixel space. The method still benefits from certificates against adversarial attacks, which is highly desirable for critical systems but lacks semantic interpretation otherwise.  

\subsection{Tuning of margin}

The algorithm is not quite agnostic to the data: the margin $m>0$ used in $\hkr$ loss is an important parameter that serves as prior on the typical distance that separates the One Class support from the anomalies. This hyper-parameter can be guessed from the final application at hand (very much like the ``scale'' parameter of radial basis function kernels), manually with grid search algorithms or more extensive procedures. Theorem~\ref{thm:hkrsdf} suggests that a small margin $m$ works best. However, the VC dimension associated with the corresponding set of classifiers increases polynomially with $\frac{1}{m}$ (see Proposition 6 of~\cite{bethune2022pay}). Hence, the algorithm benefits from faster convergence and more stability during training when $m$ is big. Fortunately, this tradeoff present in most deep learning-based algorithms is solely controlled by this one-dimensional parameter in our case. Any heuristic estimation from data or with a one-dimensional line search is feasible, even with a limited computational budget.  

\section{Hardware}\label{app:hardware}

The hardware consists of a workstation with NVIDIA 1080 GPU with 8GB memory and a machine with 32GB RAM. Hence, while being based on deep learning, and obviously benefiting from faster hardware, the requirements are affordable to most practitioners. The typical duration of an epoch on most challenging datasets was under a minute.

\end{document}